\documentclass[final,onefignum,onetabnum]{siamonline190516}

\usepackage[utf8]{inputenc} 
\usepackage[T1]{fontenc}    
\usepackage{hyperref}       
\usepackage{url}            
\usepackage{nicefrac}       
\usepackage{microtype}      
\usepackage{hyperref}
\usepackage{color}
\usepackage{latexsym}
\usepackage{dsfont}
\usepackage{float}
\usepackage[caption=false]{subfig}
\usepackage{booktabs}
\usepackage{algorithm,algorithmic}
\usepackage{graphicx}
\usepackage{amsmath, amsfonts,amssymb,theorem,euscript,array,enumerate,amsfonts,mathrsfs}
\usepackage{hhline}
\usepackage{verbatim}
\usepackage{eurosym}
\usepackage[flushleft]{threeparttable}
\usepackage[normalem]{ulem}

\renewenvironment{equation*}{\[}{\]\ignorespacesafterend}

\newsiamthm{Theorem}{Theorem}
\newsiamthm{Definition}{Definition}
\newsiamthm{Proposition}{Proposition}
\newsiamthm{Assumption}{Assumption}
\newsiamthm{Lemma}{Lemma}
\newsiamthm{Corollary}{Corollary}
\newsiamthm{Remark}{Remark}
\newsiamthm{Example}{Example}

\newcommand{\nc}{\newcommand}
\nc{\ind}{\mathds{1}}

\newcommand{\R}{\mathbb{R}}

\newcommand{\E}{\mathcal{E}}

\def\esssup_#1{\underset{#1}{\mathrm{ess\,sup\, }}}
\def\essinf_#1{\underset{#1}{\mathrm{ess\,inf\, }}}
\def\argmax_#1{\underset{#1}{\mathrm{arg\,max\, }}}
\def\argmin_#1{\underset{#1}{\mathrm{arg\,min\, }}}

\def\reff#1{{\rm(\ref{#1})}}

\def \Sum{\displaystyle\sum}

\def\b1{\bf 1}

\def \R{\mathbb{R}}

\def \E{\mathbb{E}}

\def \P{\mathbb{P}}

\def \Bc{{\cal B}}
\def \Cc{{\cal C}}

\def \Hc{{\cal H}}

\def \Pc{{\cal P}}

\def \Sc{{\cal S}}

\def \Uc{{\cal U}}

\def \Xc{{\cal X}}

\newsiamthm{Hypothesis}{Hypothesis}

\renewcommand{\theequation}{\thesection.\arabic{equation}}

\def\reff#1{{\rm(\ref{#1})}}
\def\beq{\begin{eqnarray}}
\def\enq{\end{eqnarray}}
\usepackage[mathlines]{lineno}

\author{Haotian Gu\thanks{Department of Mathematics, University of California, Berkeley, USA (\email{haotian\_gu@berkeley.edu}).}
\and Xin Guo\thanks{IEOR Department, University of California, Berkeley, USA (\email{xinguo@berkeley.edu}, \email{xiaoliwei@berkeley.edu}).}
\and Xiaoli Wei\footnotemark[3]
\and Renyuan Xu\thanks{Industrial  \& Systems Engineering, University of Southern California, Los Angeles, USA (\email{renyuanx@usc.edu}).}}

\title{Mean-Field Controls with Q-learning for Cooperative MARL:  Convergence and Complexity Analysis\thanks{Accepted August 16th, SIAM Journal on Mathematics of Data Science.
}}

\headers{MFC with Q-learning for MARL Games}{Gu, Guo, Wei and Xu}

\begin{document}

\maketitle

\begin{abstract}
Multi-agent reinforcement learning (MARL), despite its popularity and empirical success,  suffers from the curse of dimensionality. This paper  builds the mathematical framework to approximate cooperative MARL by  a mean-field control (MFC) approach, and shows that the approximation error is of $\mathcal{O}(\frac{1}{\sqrt{N}})$. By establishing an appropriate form of the dynamic programming principle for both the value function and the Q function, it proposes a model-free kernel-based Q-learning algorithm (MFC-K-Q), which is shown to have a linear convergence rate for the MFC problem, the first of its kind in the  MARL literature. It further  establishes that the convergence rate and the sample complexity of MFC-K-Q are independent of the number of agents $N$, which provides an $\mathcal{O}(\frac{1}{\sqrt{N}})$ approximation to the MARL problem with $N$ agents in the learning environment.
Empirical studies for the network traffic congestion problem demonstrate that MFC-K-Q outperforms existing  MARL algorithms when $N$ is large, for instance when $N>50$. 
\end{abstract}



\section{Introduction}
Multi-agent reinforcement learning (MARL) has enjoyed substantial successes for analyzing the otherwise challenging games, including two-agent or two-team computer games \cite{SHMGSVSAPL2016,VBCMJCDHGP2019}, self-driving vehicles \cite{SSS2016}, real-time bidding games \cite{JSLGWZ2018}, ride-sharing \cite{LQJYWWWY2019}, and traffic routing \cite{EAA2013}.
 Despite its empirical success, MARL suffers from the curse of dimensionality known also as the {\it combinatorial nature} of MARL: its sample complexity by  existing algorithms for stochastic dynamics grows exponentially with respect to the number of agents $N$.
(See \cite{HPT2019} and also Proposition \ref{lemma:N_agent_complexity} in Section \ref{sec:set-up}). In practice, this $N$ can be on the scale of thousands or more, for instance, in rider match-up for Uber-pool and network routing for Zoom.
 
One classical approach to tackle this curse of dimensionality is to focus on {\it local policies}, namely by exploiting special structures of MARL problems and by designing problem-dependent algorithms to reduce the complexity. For instance,  \cite{LR2000}  developed value-based distributed Q-learning algorithm for deterministic and finite Markov decision problems (MDPs), and \cite{QWL2019} exploited special dependence structures among agents.
(See the reviews by  \cite{yang2020overview} and \cite{ZYB2019}  and the references therein).

Another approach is to consider  MARL in the regime with a large number of homogeneous agents.  In this paradigm, by  functional strong law of large numbers (a.k.a. propagation of chaos)  \cite{K1956, M1967, S1991, G1988}, non-cooperative MARLs can be approximated under Nash equilibrium by  mean-field games with learning, and  cooperative  MARLs can be studied under Pareto optimality by analyzing mean-field controls (MFC) with learning. 
This approach is appealing not only because  the dimension of MFC or MFG is independent of the number of agents $N$, but also because solutions of MFC/MFG (without learning) have been shown to
provide good approximations to the corresponding $N$-agent  game in terms of both game values and optimal strategies \cite{huang2007large,lasry2007mean,MP2019,SBM2018,saldi2020approximate}. 

MFG with learning has gained popularity in the reinforcement learning (RL) community \cite{fu2019,GHXZ2019, iyer2014mean,yang2018mean,yin2013learning}, with its sample complexity  shown to be similar to that of single-agent RL (\cite{fu2019, GHXZ2019}). Yet MFC with learning is by and large an uncharted field despite its potentially wide range of applications \cite{LQJYWWWY2019,lin2018efficient,wang2016multi,wiering2000multi}. The main challenge for MFC with learning is to deal with probability measure space
over the state-action space, which  
 is shown (\cite{GGWR2019}) to be the minimal space for which the Dynamic Programming Principle  will hold.
One of the open problems  for MFC with learning is therefore, as pointed out in \cite{MP2019},  to design efficient RL algorithms  on probability measure space.
 
To circumvent designing algorithms  on probability measure space, \cite{CLT2019b} proposed to add common noises to the underlying dynamics. This approach enables them to apply the standard RL theory for stochastic dynamics. Their  model-free algorithm, however, suffers from high sample complexity as illustrated in Table \ref{table:algo} below, and  with weak performance as  demonstrated in Section \ref{sec:experiments}. For special classes of linear-quadratic MFCs with stochastic dynamics, \cite{CLT2019a} explored the policy gradient method and \cite{LYWK2019}  developed an actor-critic type algorithm.

\paragraph{Our work}  This paper  builds the mathematical framework to approximate cooperative MARL by  MFCs with learning. The approximation error is shown to be of  $\mathcal{O}(\frac{1}{\sqrt{N}})$. It then identifies the minimum space on which the Dynamic Programming Principle holds, and proposes an efficient approximation algorithm (MFC-K-Q) for MFC with learning.
This model-free Q-learning-based algorithm combines the technique of kernel regression with approximated Bellman operator. The convergence rate and the sample complexity of this algorithm are shown to be independent of the number of agents
$N$, and  rely only on the size of the state-action space of the underlying single-agent dynamics (Table \ref{table:algo}). 
As far as we are aware of, there is no prior algorithm with linear convergence rate for cooperative  MARL. 

Mathematically, the DPP is established through lifting the state-action space and by aggregating the
reward and the underlying dynamics. This lifting idea has been used in previous  MFC framework (\cite{PW2016, WZ2020} without learning and \cite{GGWR2019} with learning). Our work finds  that this lifting idea is  critical for efficient algorithm design for MFC with learning: the resulting deterministic dynamics from this lifting trivialize the choice of the learning rate for the convergence analysis and significantly reduce the sample complexity.

Our experiment in Section \ref{sec:experiments} demonstrates that MFC-K-Q avoids the curse of dimensionality and outperforms both existing MARL algorithms (when $N>50$) and the MFC algorithm in \cite{CLT2019b}. 
Table \ref{table:algo} summarizes the  complexity of
our MFC-K-Q algorithm along with these relevant algorithms.
 
\begin{table}[tbhp]
  \centering
  \begin{threeparttable}
  \begin{tabular}{lllll}
    \toprule
    Work & MFC/N-agent & Method & Sample Complexity Guarantee\\
    \midrule
    Our work & MFC & Q-learning & $\Omega(T_{cov}\cdot\log(1/\delta))$\\
    \cite{CLT2019b} & MFC & Q-learning & $\Omega((T_{cov}\cdot\log(1/\delta))^l\cdot\text{poly}(\log(1/(\delta\epsilon))/\epsilon))$\\
    Vanilla N-agent & N-agent & Q-learning & $\Omega(\text{poly}((|\Xc||\Uc|)^N\cdot\log(1/(\delta\epsilon))\cdot N/\epsilon))$\\
    \cite{QWL2019} & N-agent & Actor-critic & $\Omega(\text{poly}((|\Xc||\Uc|)^{f(\log(1/\epsilon))}\cdot\log(1/\delta)\cdot N/\epsilon))$\\
    \bottomrule
  \end{tabular}
  \caption{Comparison of algorithms}  \label{table:algo}
  \end{threeparttable}
   \begin{tablenotes}
  \item $T_{cov}$ in Table \ref{table:algo} is the covering time of the exploration policy and $l=\max\{3+1/\kappa, 1/(1-\kappa)\}>4$ for some $\kappa\in(0.5,1)$. Other parameters are as in Proposition \ref{lemma:N_agent_complexity} and also in Theorem \ref{thm:sample_complexity}. Note that~\cite{QWL2019} assumed that agents interact locally through a given graph so that local policies can approximate the global one, yet $f(\log(1/\epsilon))$  can scale as $N$ for a dense graph.
  \end{tablenotes}
 
\end{table}

\paragraph{Organizations}  Section \ref{sec:set-up} introduces the set-up of cooperative MARL and MFC with learning. 
Section \ref{sec:dpp} establishes the Dynamical Programming Principle for MFC with learning. Section \ref{sec:algorithm} proposes the algorithm (MFC-K-Q) for MFC with learning, with convergence and sample complexity analysis. Section \ref{sec:mfc-theorey} is dedicated to the proof of the main theorem.  Section \ref{sec:connection} connects cooperative MARL and MFC with learning. Section \ref{sec:experiments} tests performance of MFC-K-Q in a network congestion control problem.  Finally, some future directions and discussions are provided in Section \ref{sec:discussion}. For ease of exposition, proofs for all lemmas are in the Appendix.

\paragraph{Notation} For a measurable space $(\Sc, \Bc)$, where $\Bc$ is $\sigma$-algebra on $\Sc$, denote $\mathbb{R}^{\Sc}$ for the set of all  real-valued measurable functions on $\Sc$, $\mathbb{R}^{\Sc}:=\{ f:\Sc\to\mathbb{R}  |  f \text{ is measurable}\}$. For each bounded $f\in\mathbb{R}^{\Sc}$, define the sup norm of $f$ as $||f||_\infty=\sup_{s\in\Sc}|f(s)|$. In addition, when $\Sc$ is finite, we denote $|\Sc|$ for the size of $\Sc$, and $\Pc(\Sc)$ for the set of all probability measures on $\Sc$: $\{p:p(s)\geq 0, \sum_{s\in\Sc}p(s)=1\}$, which is equivalent to the probability simplex in $\mathbb{R}^{|\Sc|}$. Moreover, in $\Pc(\Sc)$, let $d_{\Pc(\Sc)}$ be the metric induced by the $l_1$ norm: for any $u, v\in\Pc(\Sc)$, $d_{\Pc(\Sc)}(u,v)=\sum_{s\in\Sc}|u(s)-v(s)|$.  $\Pc(\Sc)$ is endowed with Borel $\sigma$-algebra induced by $l_1$ norm. $1(x \in A)$ denotes the indicator function, i.e., $1(x \in A) = 1$ if $x \in A$, and $1(x \notin A) = 0$ if $x \notin A$. 
 
\section{MARL and MFC with Learning}\label{sec:set-up}
\subsection{MARL and its Complexity}\label{subsec:nplayer}

We first recall cooperative MARL in an infinite time horizon, where there are $N$ agents whose game strategies are coordinated by a central controller. Let us assume the state space $\mathcal{X}$ and the action space $\mathcal{U}$ are all finite.

At each step $t=0,1, \cdots,$ the state of agent $j$ $(= 1, 2, \cdots , N)$ is $x_t^{j, N} \in \mathcal{X}$ and she takes an action $u_t^{j, N} \in \mathcal{U}$.  Given the current state profile $\pmb{x}_t = (x_t^{1, N},\cdots,x_t^{N, N})\in \mathcal{X}^N$ and the current action profile $\pmb{u}_t = (u_t^{1, N},\cdots,u_t^{N, N})\in \mathcal{U}^N$ of $N$-agents, agent $j$ will receive a reward $\tilde r^j({\pmb x_t}, {\pmb u_t})$ and her state will change to $x_{t + 1}^{j, N}$ according to a transition probability function $P^j({\pmb x_t}, {\pmb u_t})$. A Markovian game further restricts the admissible policy for agent $j$ to be of the form $u_t^{j, N} \sim \pi_t^j(\pmb{x}_t)$. That is, $\pi_t^j:  \mathcal{X}^N \rightarrow \mathcal{P}(\mathcal{U})$ maps each state profile $\pmb{x}\in \mathcal{X}^N$ to a randomized action, with $\mathcal{P}(\mathcal{U} )$ the probability measure space on space $\mathcal{U}$.

In this cooperative MARL, the central controller is to maximize the expected discounted aggregated accumulated rewards over all  policies and averaged over all agents. That is  to find 

\begin{eqnarray*} 
\sup_{{\pmb \pi}}\frac{1}{N} \sum_{j=1}^N v^j({\pmb x}, {\pmb \pi}),\text{  where }v^j({\pmb x}, {\pmb \pi}) = \mathbb{E} \biggl[\sum_{t=0}^\infty \gamma^t \tilde r^j({\pmb x_t}, {\pmb u_t}) \big| {\pmb x_0} = {\pmb x}\biggl]
\end{eqnarray*}
is the accumulated reward for agent $j$, given the initial state profile $\pmb{x}_0 = \pmb{x}$ and policy ${\pmb \pi} = \{{\pmb \pi}_t\}_{t = 0}^\infty$ with ${\pmb \pi}_t=(\pi_t^1, \ldots, \pi_t^N)$. Here 
 $\gamma$ $\in$ $(0, 1)$ is a discount factor, {$u_t^{j, N} \sim \pi_t^j(\pmb{x}_t)$, and $x_{t+1}^{j, N}\sim P^j({\pmb x_t}, {\pmb u_t})$}. 

 The sample complexity of the Q learning algorithm of this cooperative MARL is  exponential with respect to $N$. Indeed, take Theorem $4$ in~\cite{even2003learning} and note that the corresponding covering time for the policy of the central controller will be at least $(|\Xc||\Uc|)^N$, then we see 
\begin{Proposition}\label{lemma:N_agent_complexity}
Let $|\Xc|$ and $|\Uc|$ be respectively the size of the state space $\Xc$ and the action space $\Uc$. Let $Q^*$ and  $Q_T$ be respectively the optimal value and the value of the asynchronous Q-learning algorithm in~\cite{even2003learning} using polynomial learning rate with time $T=\Omega\bigg(\text{poly}\bigg((|\Xc||\Uc|)^N\cdot\frac{N}{\epsilon}\cdot\ln(\frac{1}{\delta\epsilon})\bigg)\bigg)$. Then with probability at least $1-\delta$, $\|Q_T-Q^*\|_\infty \leq\epsilon$.
\end{Proposition}

This exponential growth in sample complexity makes the algorithm difficult to scale up. 
The classical approach for this curse of dimensionality is to explore  special network structures (e.g., sparsity or local interactions among agents) for MARL problems. 
Here we shall  propose an alternative approach in the regime when  there is a large number of homogeneous agents.

\subsection{MFC with Learning: Set-up, Assumptions and Some Preliminary Results}
To overcome the curse of dimensionality in $N$, we now propose a   mean-field control (MFC) framework to approximate this cooperative MARL when agents are homogeneous. 

In this MFC framework, all agents are assumed to be identical, indistinguishable, and interchangeable, and each agent $j (=1, \cdots, N)$  is assumed to depend on all other agents only through the empirical distribution of their states and actions. 
That is, denote $\mathcal{P}(\mathcal{X})$ and $\mathcal{P}(\mathcal{U})$ as the probability measure spaces over the state space $\mathcal{X}$ and the action space $\mathcal{U}$, respectively. The empirical distribution of the states is $\mu^{N}_t(x)$ $=$ $\frac{ \sum_{j=1}^N 1({x^{j, N}_t} = x)}{N}$ {{$\in \Pc(\Xc)$}},  and the empirical distribution of the actions is $\nu^{N}_t(u)$ $=$ $\frac{\sum_{j=1}^N 1({u^{j, N}_t}=u)}{N} ${{$\in \Pc(\Uc)$}}.
Then, by law of large numbers, this coperative MARL becomes an MFC with learning  when $N \rightarrow \infty$. 
Moreover, as all agents are indistinguishable, one can focus on a single representative agent.

Mathematically,  this MFC  with learning is as follows. At each time $t=0,1, \cdots,$ the representative agent in state $x_t$ takes an action $u_t$ $\in$ $\mathcal{U}$ according to the admissible policy $\pi_t(x_t,\mu_t): \mathcal{X} \times \mathcal{P}(\mathcal{X}) \to \mathcal{P}(\mathcal{U})$ assigned by the central controller, who can observe the population state distribution $\mu_t \in \mathcal{P}(\mathcal{X})$. Further denote $\Pi:=\{\pi=\{\pi_t\}_{t=0}^{\infty}\vert \pi_t:\mathcal{X} \times \mathcal{P}(\mathcal{X}) \to \mathcal{P}(\mathcal{U}) \text{ is measurable}\}$ as the set of admissible policies. The agent will then receive a reward $\tilde r(x_t, \mu_t, u_t, \nu_t)$ and move to the next state $x_{t + 1}$ $\in$ $\mathcal{X}$ according to a probability transition function $P(x_t, \mu_t, u_t, \nu_t)$. Here $P$ and $\tilde r$ rely on the state distribution $\mu_t$ and the action distribution $\nu_t(\cdot):=\sum_{x \in \mathcal{X}} \pi_t(x, \mu_t)(\cdot)\mu_t(x)$, and are possibly unknown.

The objective for this MFC  with learning is  to find $v$ the maximal expected discounted accumulated reward  over all admissible  policies $\pi = \{\pi_t\}_{t=0}^\infty$, namely
\begin{linenomath}
\begin{align}  \label{mfc_objective_2}
v(\mu)&=\sup_{\pi \in \Pi} v^\pi(\mu)  : =\sup_{\pi \in \Pi}  \mathbb{E}\biggl[ \sum_{t = 0}^\infty \gamma^t \tilde r(x_t, \mu_t, u_t, \nu_t) \left| x_0 \sim \mu \biggl],\right. \tag{MFC}\\
\text {subject to}\ \  & \; x_{t + 1} \sim P(x_t, \mu_t, u_t, \nu_t),\;\;\;  u_t \sim \pi_t(x_t, \mu_t). \nonumber
\end{align}
\end{linenomath}
with initial condition $\mu_0=\mu$.


Note that after observing $\mu_t$, the policy from the central controller $\pi_t(\cdot,\mu_t)$ can be viewed as a mapping from $\Xc$ to $\Pc(\Uc)$. In this case, we set 
\begin{equation} \label{localpolicyh}
 h_t(\cdot): = \pi_t(\cdot, \mu_t)
\end{equation}
for notation simplicity and denote $\Hc:=\{h:\Xc\to\Pc(\Uc)\}$  as the space for $h_t(\cdot)$. Note that  $\mathcal{H}$ is isomorphic to the product of $|\Xc|$ copies of $\Pc(\Uc)$. Therefore, the set of admissible policies $\Pi$ can be rewritten as  
\beq \label{admissiblePi}
\Pi: = \Big\{\pi=\{\pi_t\}_{t=0}^\infty \,|\, \pi_t:\Pc(\Xc)\to\Hc \text{ is measurable}\Big\}.
\enq 
This reformulation of the admissible policy set is key for  deriving the Dynamic Programming Principle (DPP) of \reff{mfc_objective_2}: it enables us to show that the objective in \reff{mfc_objective_2} is law-invariant and the probability distribution of the dynamics in \reff{mfc_objective_2} satisfies flow property.  This flow property is also crucial for establishing the convergence of the associated cooperative MARL by \reff{mfc_objective_2}.

\begin{Lemma}  \label{lemma:flowmut} 
Under any admissible policy $\pi = \{\pi_t\}_{t=0}^\infty \in \Pi$, and the initial state distribution $x_0\sim\mu_0=\mu$, the evolution of the state distribution $\{\mu_t\}_{t\geq0}$,  is given by
\beq \label{eqv:flowmut}
\mu_{t + 1}  &=& \Phi(\mu_t, h_t),
\enq
where $h_t(\cdot)$ is defined in \reff{localpolicyh}  and the dynamics $\Phi$ is defined as
\beq \label{equ:Phi}
\Phi(\mu, h) &:=& \sum_{x \in \mathcal{X}}\sum_{u \in \mathcal{U}} P(x, \mu, u, {\nu}(\mu, h)) \mu(x) h(x)(u) \in \Pc(\Xc),
\enq
for any $(\mu, h) \in \Pc(\Xc) \times \Hc$ and ${\nu}(\mu, h)(\cdot) := \sum_{x \in \mathcal{X}} h(x)(\cdot)\mu(x) \in \Pc(\Uc)$. Moreover,  the value function $v^{\pi}$ defined in \reff{mfc_objective_2} can be rewritten as 
\beq \label{equ:reformv}
v^{\pi}(\mu) &=& \sum_{t=0}^{\infty} \gamma^{t} r(\mu_t, h_t),
\enq
where for any $(\mu, h) \in \Pc(\Xc) \times \Hc$, the reward $r$ is defined as
\beq \label{equ: r}
r(\mu, h)&:=&\sum_{x \in \mathcal{X}} \sum_{u \in \mathcal{U}} \tilde r(x, \mu, u, {\nu}(\mu, h)) \mu(x) h(x)(u).
\label{equ:r}
\enq
\end{Lemma}
\begin{Remark}
\label{remark:aggregation}
Because of the aggregated forms of $\Phi$ and $r$ from \reff{equ:Phi} and \reff{equ: r}, they are also called the aggregated dynamics and the aggregated reward, respectively.

\end{Remark}


We start with some
standard regularity assumptions for  MFC problems \cite{CD2018}. These assumptions are necessary  for  the  mean-field approximation to cooperative MARL and  for the subsequent convergence and sample complexity analysis of the learning algorithm.


Let us use the $l_1$ distance for the metrics $d_{\Pc(\Xc)}$ and $d_{\Pc(\Uc)}$ of $\Pc(\Xc)$ and $\Pc(\Uc)$, and define  $d_\Hc(h_1,h_2)=\max_{x\in\Xc}||h_1(x)-h_2(x)||_1$ and $d_\Cc((\mu_1,h_1),(\mu_2,h_2))=||\mu_1-\mu_2||_1+d_\Hc(h_1,h_2)$ for the space  $\Hc$ and $\Cc:= \Pc(\Xc) \times \Hc$, respectively. Moreover, we endow $\Cc$ with Borel $\sigma$ algebra generated by open sets in $d_{\Cc}$.

\begin{Assumption}[Continuity and boundedness of $\tilde{r}$]\label{ass:r_MFC}
    There exist $\tilde{R}>0, L_{\tilde r}>0$, such that for all $x\in\Xc,u\in\Uc$, $\mu_1,\mu_2\in\Pc(\Xc), \nu_1,\nu_2\in\Pc(\mathcal{U})$,
    \begin{eqnarray*}
    |\tilde{r}(x,\mu_1,u,\nu_1)|\leq\tilde{R}, \; |\tilde{r}(x,\mu_1,u,\nu_1)-\tilde{r}(x,\mu_2,u,\nu_2)|\leq L_{\tilde{r}}\cdot(||\mu_1-\mu_2||_1+||\nu_1-\nu_2||_1).
    \end{eqnarray*}
\end{Assumption}

\begin{Assumption}[Continuity of $P$]\label{ass:P_MFC}
    There exists $L_P>0$ such that for all $x\in\Xc, u\in\Uc, \mu_1,\mu_2\in\Pc(\Xc), \nu_1,\nu_2\in\Pc(\Uc),$
    
   $ ||P(x,\mu_1,u,\nu_1)-P(x,\mu_2,u,\nu_2)||_1\leq L_{P} \cdot (||\mu_1-\mu_2||_1+||\nu_1-\nu_2||_1).$
\end{Assumption}

 Note that $l_1$ distance between transition kernels $P(x, \mu, u, \nu)$ in Assumption \ref{ass:P_MFC} is equivalent to 1-Wasserstein distance  when $\Xc$ and $\Uc$ are equipped with discrete metrics $1{(x_1 \neq x_2)}$ for $x_1, x_2 \in \Xc$ and $1{(u_1 \neq u_2)}$ for $u_1, u_2 \in \Uc$, respectively, see e.g., \cite{GS2002}, \cite{Hinderer2005}. Under Assumptions \ref{ass:r_MFC} and \ref{ass:P_MFC}, it is clear that the probability measure  $\nu$ over the action space, the aggregated reward $r$ in \reff{equ:r}, and the aggregated dynamics $\Phi$ in \reff{equ:Phi} are all Lipschitz continuous, which  will be useful for subsequent analysis. 
\begin{Lemma}[Continuity of $\nu$]\label{lemma:cont_nu}
\begin{equation}
    \|\nu(\mu,h)-\nu(\mu',h')\|_1\leq d_\Cc((\mu,h),(\mu',h')).
\end{equation}
\end{Lemma}

\begin{Lemma}[Continuity of $r$]\label{lemma:cont_r}
Under Assumption~\ref{ass:r_MFC},
\begin{equation}
    |r(\mu,h)-r(\mu',h')|\leq (\tilde{R}+2L_{\tilde{r}})d_\Cc((\mu,h),(\mu',h')).
\end{equation}
\end{Lemma}

\begin{Lemma}[Continuity of $\Phi$]\label{lemma:cont_phi}
Under Assumption~\ref{ass:P_MFC},
\begin{equation}
    \|\Phi(\mu,h)-\Phi(\mu',h')\|_1\leq(2L_{P}+1)d_\Cc((\mu,h),(\mu',h')).
\end{equation}
\end{Lemma}

\section{DPP for Q Function in MFC with learning}\label{sec:dpp}
In this section, we establish the DPP of the Q function for \reff{mfc_objective_2}. Different  from the well-understood DPP for single-agent control problem  (see for example \cite[chapter 9]{meyn2008control} and \cite{meyn1999algorithms}), DPP for mean-field control problem has been established only recently on the lifted probability measure space
\cite{GGWR2019,PW2016,WZ2020}. We extend the approach of \cite{GGWR2019} to allow $P$ and $\tilde{r}$ to depend on the population's action distribution $\nu_t$.

First, by Lemma \ref{lemma:flowmut}, \reff{mfc_objective_2}   can be recast as a general Markov decision  problem (MDP)
with probability measure space as the new state-action space. More specifically, recall the set of admissible policies $\Pi$ in \reff{admissiblePi}, if one views the policy $\pi_t$ to be a mapping from $\Pc(\Xc)$ to $\Hc$,  then  \reff{mfc_objective_2}  can be restated as the following MDP with unknown  $r$ and $\Phi$:
\begin{linenomath}
\begin{align}\label{mfc_objective_1}
 & v(\mu) := \sup_{\pi\in\Pi}\sum_{t=0}^{\infty} \gamma^{t} r(\mu_t, h_t) \tag{MDP}\\
\text {subject to} \ \ \ \ &
   \mu_{t + 1}=\Phi(\mu_t, h_t), \ \ \mu_0=\mu,  {\rm and }\,\,h_t(\cdot) \text{ in \reff{localpolicyh}}.\nonumber
\end{align}
\end{linenomath}
With this reformulation, we can  define the associated optimal Q function for  \eqref{mfc_objective_1} starting from arbitrary $(\mu, h) \in \Cc = \Pc(\Xc) \times \Hc$,
\beq \label{equmfcQ}
Q(\mu, h) := \sup_{\pi\in\Pi}\bigg[\sum_{t=0}^{\infty} \gamma^{t} r(\mu_t, h_t)\bigg|\mu_0=\mu, \pi_0(\mu_0)=h\bigg],
\enq
with $h_t(\cdot)$ defined in \reff{localpolicyh}. Similarly, define $Q^\pi$ as the Q function associated with a  policy $\pi$:
\beq \label{equ:Q_pi}
Q^\pi(\mu, h) := \bigg[\sum_{t=0}^{\infty} \gamma^{t} r(\mu_t, h_t)\bigg|\mu_0=\mu, \pi_0(\mu_0)=h\bigg],
\enq
with $h_t(\cdot)$ defined in \reff{localpolicyh}.
\begin{Remark}
With this reformulation,  \reff{mfc_objective_2} is now lifted from the finite state-action space $\cal X$ and $\cal U$ to a compact continuous state-action space $\Cc$ embedded in an Euclidean space. In addition, the dynamics become {\it deterministic} by the aggregation over the original state-action space.
Due to this aggregation for $r$, $\Phi$, and the Q function, we will subsequently refer this Q in \reff{equmfcQ}  as an Integrated Q (IQ) function, to underline the difference between the Q function for RL of single agent and that for MFC with learning. 

\end{Remark}

The following theorem shows Bellman equation for the IQ function in \reff{equmfcQ}.
\begin{Theorem}\label{thm:MKC=MDP}
For any $\mu$ $\in$ $\Pc(\Xc)$,  
\beq \label{equ:relationvQ}
v(\mu) &=& \sup_{h \in \Hc} Q(\mu, h) = \sup_{h \in \Hc}\sup_{\pi\in\Pi}Q^\pi(\mu, h).
\enq
Moreover, the Bellman equation for  $Q$ $: \Cc \to \R$ is
\begin{equation} \label{BellmanQ}
Q(\mu, h) = {r}(\mu, h) + \gamma \sup_{\tilde h \in \Hc} Q(\Phi(\mu, h), \tilde h).
\end{equation}
\end{Theorem}
\begin{proof}[Proof of Theorem \ref{thm:MKC=MDP}]
Recall the definition of $v$ in \eqref{mfc_objective_1} and $Q$ in \eqref{equmfcQ}. For $v(\mu)$, the supremum is taken over all the admissible policies $\Pi$, while for $Q(\mu, h)$, the supremum is taken over all the admissible policies $\Pi$ with a further restriction that $\pi_0(\mu)=h$. Now in $\sup_{h \in \Hc} Q(\mu, h)$, since we are free to choose $h$, it is equivalent to $v$. Moreover, 
\begin{eqnarray*}
v(\mu) &=& \sup_{\pi\in\Pi} \biggl[ \sum_{t=0}^{\infty} \gamma^{t} r(\mu_t, \pi_t(\mu_t))\bigg|\mu_0 = \mu\biggl]= \sup_{\pi\in\Pi, \pi_0(\mu)=h, h\in\Hc}\bigg[\sum_{t=0}^{\infty} \gamma^{t} r(\mu_t, \pi_t(\mu_t))\bigg|\mu_0=\mu, \pi_0(\mu_0)=h\bigg]\\
&=& \sup_{h\in\Hc}\sup_{\pi\in\Pi, \pi_0(\mu)=h}\bigg[\sum_{t=0}^{\infty} \gamma^{t} r(\mu_t, \pi_t(\mu_t))\bigg|\mu_0=\mu, \pi_0(\mu_0)=h\bigg]= \sup_{h \in \Hc} Q(\mu, h).
\end{eqnarray*}

\begin{eqnarray*}
Q(\mu, h) &=& \sup_{\pi\in\Pi}\bigg[\sum_{t=0}^{\infty} \gamma^{t} r(\mu_t, \pi_t(\mu_t))\bigg|\mu_0=\mu, \pi_0(\mu_0)=h\bigg] \\
&=& r(\mu, h) + \sup_{\{\pi_t\}_{t=1}^\infty}\bigg[\sum_{t=1}^{\infty} \gamma^{t} r(\mu_t, \pi_t(\mu_t))\bigg|\mu_1=\Phi(\mu, h)\bigg] \\
&=& r(\mu, h) + \sup_{\{\pi_t\}_{t=0}^\infty}\gamma\bigg[\sum_{t=0}^{\infty} \gamma^{t} r(\mu_t, \pi_t(\mu_t))\bigg|\mu_0=\Phi(\mu, h)\bigg] \\
&= &r(\mu, h) + \gamma v(\Phi(\mu, h))= r(\mu, h) + \gamma \sup_{h \in \Hc} Q(\Phi(\mu, h), h),
\end{eqnarray*}
where the third equality is from shifting the time index by one.
\end{proof}

Next, we have the following  verification theorem for this  IQ function.
\begin{Proposition}[Verification]\label{prop:veri}
Assume Assumption \ref{ass:r_MFC} and define $V_{\max}:=\frac{R}{1-\gamma}$. Then,
\begin{itemize}
\setlength{\itemindent}{.5in}
    \item $Q$ defined in \eqref{equmfcQ} is the unique function in $\{f\in\R^{\Cc}:\|f\|_\infty\leq V_{\max}\}$ satisfying the Bellman equation \eqref{BellmanQ}. 
    \item Suppose that for every $\mu$ $\in$ $\Pc(\Xc)$, one can find an $h_\mu\in\Hc$ such that $h_\mu \in \arg\max_{h\in\Hc}Q(\mu, h)$, then $\pi^*=\{\pi_t^*\}_{t=0}^\infty$, where $\pi_t^*(\mu)=h_\mu$ for any $\mu\in\Pc(\Xc)$ and $t\geq0$, is an optimal stationary policy of  \reff{mfc_objective_1}.
\end{itemize}
\end{Proposition}

In order to prove the proposition, let us first define the following two operators.
\begin{itemize}
    \item Define the operator $B:\mathbb{R}^{\Cc}\to\mathbb{R}^{\Cc}$ for  \eqref{mfc_objective_1}
    \begin{equation}\label{equ:bellman_original}
        ({B}\,q)(c)={r}(c)+\gamma\max_{\tilde{h}\in\Hc}q({\Phi}(c),\tilde{h}).
    \end{equation}
    \item Define the operator $B^\pi:\mathbb{R}^{\Cc}\to\mathbb{R}^{\Cc}$ for \eqref{mfc_objective_1} under a given stationary policy $\{\pi_t=\pi:\Pc(\Xc)\to\Hc\}_{t=0}^\infty$
    \begin{equation}\label{equ:bellman_pi}
        ({B^\pi}\,q)(c)={r}(c)+\gamma q({\Phi}(c),\pi(\Phi(c))).
    \end{equation}
\end{itemize}
\begin{proof}
     Since $||\tilde r||_\infty\leq R$, for any $\mu\in\Pc(\Xc)$ and $h\in\Hc$, the aggregated reward function \eqref{equ:r} satisfies $|r(\mu, h)|\leq R\cdot\sum_{x\in\Xc}\sum_{u\in\Uc}\mu(x)h(x)(u)=R.$
    In this case, for any $\mu\in\Pc(\Xc)$, $h\in\Hc$ and policy $\pi$, $|Q^\pi(\mu, h)|\leq R\cdot\sum_{t=0}^\infty \gamma^t=V_{\max}$. Hence, $Q$ of \eqref{equmfcQ} and $Q^\pi$ of \eqref{equ:Q_pi} both belong to $\{f\in\R^{\Cc}:\|f\|_\infty\leq V_{\max}\}$. Meanwhile, by definition, it is easy to show that $B$ and $B^\pi$ map $\{f\in\R^{\Cc}:\|f\|_\infty\leq V_{\max}\}$ to itself.
    
    Next, we notice that $B$ is a contraction operator with modulus $\gamma<1$ under the sup norm on $\{f\in\R^{\Cc}:\|f\|_\infty\leq V_{\max}\}$: for any $(\mu,h)\in\Cc$,
    \begin{eqnarray*}
       |B q_1(\mu,h)-B q_2(\mu,h)| \leq \gamma \max_{\tilde{h}\in\Hc} |q_1(\Phi(\mu,h),\tilde{h})-q_2(\Phi(\mu,h),\tilde{h})| \leq \gamma \|q_1-q_2\|_\infty.
    \end{eqnarray*}
     Thus, $\|B q_1-B q_2\|_\infty\leq\gamma \|q_1-q_2\|_\infty$. By Banach Fixed Point Theorem, $B$ has a unique fixed point in $\{f\in\R^{\Cc}:\|f\|_\infty\leq V_{\max}\}$. By \eqref{BellmanQ} in Theorem \ref{thm:MKC=MDP}, the unique fixed point is $Q$.
    
    Similarly, we can show that for any stationary policy $\pi$, $B^\pi$ is also a contraction operator with modulus $\gamma<1$. Meanwhile, by the standard DPP argument as in Theorem \ref{thm:MKC=MDP}, we have $Q^{\pi}=B^{\pi}Q^{\pi}$. This implies $Q^{\pi}$ is the unique fixed point for $B^{\pi}$ in $\{f\in\R^{\Cc}:\|f\|_\infty\leq V_{\max}\}$.
    
    Now let $\pi^*$ be the stationary policy defined in the statement of Proposition \ref{prop:veri}. By definition, for any $c\in\Cc$, $Q(c)=r(c)+\gamma\max_{\tilde{h}\in\Hc}Q(\Phi(c),
    \tilde{h})=r(c)+\gamma Q(\Phi(c),
    \pi^*(\Phi(c)))=B^{\pi^*}Q(c)$.
    Since $B^{\pi^*}$ has a unique fixed point $Q^{\pi^*}$ in $\{f\in\R^{\Cc}:\|f\|_\infty\leq V_{\max}\}$, which is the IQ function for the stationary policy $\pi^*$, clearly $Q^{\pi^*}=Q$, and the optimal IQ function is attained by the optimal policy $\pi^*$.
\end{proof}

\begin{lemma}[Characterization of  $Q$]\label{continuity-Qc}
Assume Assumptions~\ref{ass:r_MFC} and ~\ref{ass:P_MFC}, and $\gamma\cdot (2L_P + 1) < 1$. $Q$ of \eqref{equmfcQ} is continuous.
\end{lemma}

The continuity property of  $Q$ from Lemma \ref{continuity-Qc}, along with the compactness of $\Hc$ and Proposition \ref{prop:veri}, leads to the following existence of stationary optimal policy. 
\begin{Lemma}\label{lemma:exist_policy}
    Assume Assumptions~\ref{ass:r_MFC},~\ref{ass:P_MFC} and $\gamma\cdot (2L_P + 1) < 1$. There exists an optimal stationary policy $\pi^*:\Pc(\Xc)\to\Hc$ such that $Q^{\pi^*}=Q$.
\end{Lemma}

This existence of a stationary optimal policy is essential for the convergence analysis of our algorithm MFC-K-Q in Algorithm \ref{QL_CD}. In particular, it allows for comparing the optimal values of two MDPs with different action spaces: \eqref{mfc_objective_1} and its variant defined in \eqref{mfc_objective_dis}-\eqref{mfc_dynamics_dis}.

Note that the existence of a stationary optimal policy is well known  when the state and action spaces are {\it finite} (see for example \cite{sutton2018reinforcement}) or {\it countably infinite} (see for example \cite[chapter 9]{meyn2008control}). Yet, 
we are unable to find any prior corresponding result for the case with continuous state-action space.

\section{MFC-K-Q Algorithm via Kernel Regression and Approximated Bellman Operator}\label{sec:algorithm}
In this section, we will develop a kernel-based Q-learning algorithm (MFC-K-Q) for the MFC problem with learning based on \reff{BellmanQ}.

Note from \eqref{BellmanQ}, the MFC problem with learning is different from the classical MDP \cite{sutton2018reinforcement} in two aspects. First, the lifted state space $\Pc(\Xc)$ and lifted action space $\Hc$ are continuous, rather than discrete or finite. Second, the maximum in the Bellman operator is taken over a continuous space $\Hc$. 

To handle the lifted continuous state-action space, we use a kernel regression method on the discretized state-action space. Kernel regression is a local averaging approach for approximating {\it unknown} state-action pair from  {\it observed} data on a discretized space called {\it $\epsilon$-net}. Mathematically, a set $\Cc_\epsilon=\{c^i=(\mu^i, h^i)\}_{i=1}^{N_\epsilon}$ is an $\epsilon$-net for $\Cc$ if $\min_{1 \leq i \leq N_\epsilon} d_\Cc(c, c^i)$ $<$ $\epsilon$ for all $c$ $\in$ $\Cc$. Here $N_\epsilon$ is the size of $\Cc_\epsilon$. Note that compactness of $\Cc$ implies the existence of such an $\epsilon$-net $\Cc_\epsilon$. The choice of $\epsilon$ is critical for the convergence and the sample complexity analysis.

Correspondingly, we  define the so-called {\it kernel regression operator} $\Gamma_K:\mathbb{R}^{\Cc_\epsilon}\to\mathbb{R}^{\Cc}$:
\begin{eqnarray}\label{kernel_def}
\Gamma_Kf(c)=\Sum_{i=1}^{N_\epsilon}K(c^i,c)f(c^i),
\end{eqnarray}
where $K({c^i,c})$ $\geq$ $0$ is a weighted kernel function such that for all $c \in \mathcal{C}$ and $c^i \in \mathcal{C}_{\epsilon}$,
\begin{equation}\label{equ:kernel_cond}
    \sum_{i=1}^{N_{\epsilon}}K(c^i,c)=1\text{, and }K(c^i,c)=0\text{ if }d_{\mathcal{C}}(c^i,c)>\epsilon.
\end{equation}
In fact, $K$ can be of any form
   $ K(c^i, c) = \frac{\phi(c^i, c)}{\sum_{i=1}^{N_\epsilon} \phi(c^i, c)}$,
with some function $\phi$ satisfying $\phi\geq0$ and $\phi(x,y)=0$ when $d_{\mathcal{C}}(x,y)\geq \epsilon$. (See Section \ref{sec:experiments} for some choices of  $\phi$).

Meanwhile, to avoid maximizing over a continuous space $\Hc$ as in the Bellman equation \reff{BellmanQ}, we take the maximum over the $\epsilon$-net $\Hc_\epsilon$ on $\Hc$. Here $\Hc_\epsilon$ is an $\epsilon$-net on $\Hc$ induced from $\Cc_\epsilon$, i.e., ${\Hc}_\epsilon$ contains all the possible action choices in $\Cc_{\epsilon}$, whose size is denoted by $N_{\Hc_\epsilon}$.

The corresponding approximated Bellman operator $B_\epsilon$ acting on functions is then defined on the $\epsilon$-net $\Cc_{\epsilon}$: $\mathbb{R}^{\Cc_\epsilon}\to\mathbb{R}^{\Cc_\epsilon}$ such that
\begin{equation}\label{equ:bellman_kernel}
    (B_\epsilon\,q)(c^i)=r(c^i)+\gamma\max_{\tilde{h}\in\Hc_{\epsilon}}\Gamma_Kq(\Phi(c^i),\tilde{h}).
\end{equation}
Since $(\Phi(c^i), \tilde{h})$ may not be on the $\epsilon$-net, one needs to approximate the value at that point via the kernel regression $\Gamma_Kq(\Phi(c^i),\tilde{h})$.

In practice, one may only have access to noisy estimations $\{\widehat{r}(c^i),\widehat{\Phi}(c^i)\}_{i=1}^{N_\epsilon}$ instead of the accurate data $\{r(c^i),\Phi(c^i)\}_{i=1}^{N_\epsilon}$ on $\Cc_\epsilon$. Taking this into consideration, Algorithm \ref{QL_CD} consists of two steps. First,  it collects samples on $\Cc$ given an exploration policy. For each component $c^i$ on the $\epsilon$-net $\Cc_\epsilon$, the estimated data $(\widehat{r}(c^i),\widehat{\Phi}(c^i))$ is computed by averaging samples in the $\epsilon$-neighborhood of $c^i$. Second, the fixed point iteration is applied to the approximated Bellman operator $B_\epsilon$ with $\{\widehat{r}(c^i),\widehat{\Phi}(c^i)\}_{i=1}^{N_\epsilon}$. Under appropriate conditions, Algorithm \ref{QL_CD} provides an accurate estimation of the true Q function with efficient sample complexity (See Theorem \ref{thm:conv_mfc}). 

\begin{algorithm}[!ht]
  \caption{\textbf{Kernel-based Q-learning Algorithm for MFC (MFC-K-Q)}}
  \label{QL_CD}
\begin{algorithmic}[1]
    \STATE \textbf{Input}: Initial state distribution $\mu_0$, $\epsilon>0$, $\epsilon$-net on $\Cc:\Cc_{\epsilon}=\{c^i=(\mu^i, h^i)\}_{i=1}^{N_\epsilon}$, exploration policy $\pi$ taking actions from $\Hc_\epsilon$ induced from $\Cc_{\epsilon}$, regression kernel $K$ on $\Cc_{\epsilon}$.
    \STATE \textbf{Initialize}: $\widehat{r}(c^i)=0$, $\widehat{\Phi}(c^i)=0$, $N(c^i)=0, \forall i$.
    \REPEAT
        \STATE At the current state {distribution} $\mu_t$, act $h_t$ according to $\pi$, observe $\mu_{t+1}=\Phi(\mu_t,h_t)$ and $r_t=r(\mu_t,h_t)$.
        \FOR {${1 \leq i \leq N_\epsilon}$}
            \IF {$d_\Cc(c^i, (\mu_t,h_t))<\epsilon$} 
                \STATE $N(c^i) {\leftarrow}N(c^i)+1$.
                \STATE $\widehat{r}(c^i){\leftarrow}\frac{N(c^i)-1}{N(c^i)}\cdot\widehat{r}(c^i)+\frac{1}{N(c^i)}\cdot r_t$
                \STATE $\widehat{\Phi}(c^i){\leftarrow}\frac{N(c^i)-1}{N(c^i)}\cdot\widehat{\Phi}(c^i)+\frac{1}{N(c^i)}\cdot \mu_t$
            \ENDIF
        \ENDFOR
    \UNTIL{$N(c^i)>0, \forall i$}.
    \STATE \textbf{Initialize}: $\widehat{q}_0(c^i)=0,\forall c^i\in\Cc_{\epsilon}$, $l=0$.
    \REPEAT
        \FOR {$c^i\in\Cc_{\epsilon}$}
        {
            \STATE $\widehat{q}_{l+1}(c^i) {\leftarrow}  \Big(\widehat{r}(c^i)+$ $\gamma\max_{\tilde{h}\in\Hc_\epsilon}\Gamma_K\widehat{q}_l(\widehat{\Phi}(c^i),\tilde{h})\Big)$.
        }
        \ENDFOR
        \STATE $l=l+1$.
    \UNTIL{converge}
\end{algorithmic}
\end{algorithm}

\section{Convergence and Sample Complexity Analysis of MFC-K-Q}\label{sec:mfc-theorey}
In this section, we will establish the convergence of MFC-K-Q algorithm and analyze its sample complexity. 
The convergence analysis in Section \ref{subsec:MFC_convergence} relies on studying the  fixed point iteration of $B_\epsilon$; and  the complexity analysis in Section \ref{subsec:MFC_complexity} is   based on an upper bound of the necessary sample size to visit each $\epsilon$-neighborhood of the $\epsilon$-net at least once.


In addition to Assumptions~\ref{ass:r_MFC} and~\ref{ass:P_MFC}, the following conditions are needed for  the convergence and the sample complexity analysis. 

\begin{Assumption}[Controllability of the dynamics]\label{ass:dynamic}
For all $\epsilon$, there exists $M_\epsilon\in\mathbb{N}$ {such that} for any $\epsilon$-net $\Hc_\epsilon$ on $\Hc$ { and } $\mu,\mu'\in\Pc(\Xc)$, there exists an action sequence $(h^1,\dots, h^m)$
with $h^i\in\Hc_\epsilon$ and $m<M_\epsilon$, with which  the state  $\mu$ will be driven to an $\epsilon$-neighborhood of $\mu'$.
\end{Assumption}

\begin{Assumption}[Regularity of kernels]\label{ass:kernel}
For any point $c\in\Cc$, there exist at most $N_{K}$ points $c^i$'s in $\Cc_\epsilon$ such that $K(c^i,c)>0$. Moreover, there exists an $L_K>0$ such that for all $c \in \Cc_\epsilon, c',c''\in\Cc, |K(c,c')-K(c,c'')|\leq L_K\cdot d_{\Cc}(c',c'')$.
\end{Assumption}

Assumption~\ref{ass:dynamic} ensures the dynamics to be controllable. Assumption~\ref{ass:kernel} is easy to be satisfied: take a uniform grid as the $\epsilon$-net, then $N_K$ is roughly bounded from above by $2^{\text{dim}(\Cc)}$;
meanwhile, a number of commonly used kernels, including  the triangular kernel in Section \ref{sec:experiments}, satisfy the Lipschitz condition in 
Assumption~\ref{ass:kernel}.


\subsection{Convergence Analysis} \label{subsec:MFC_convergence}
To start, recall  the Lipschitz continuity of the aggregated rewards $r$ and dynamics $\Phi$ from Lemma \ref{lemma:cont_r} and Lemma \ref{lemma:cont_phi}. To simplify the notation,  denote $L_r:=\tilde{R}+2L_{\tilde{r}}$ as the Lipschitz constant of $r$ and  $L_\Phi:=2L_P+1$ as the Lipschitz constant of $\Phi$.

Next, recall  that there are three sources of the approximation error in Algorithm \ref{QL_CD}: the kernel regression $\Gamma_K$ on $\Cc$ with the $\epsilon$-net $\Cc_\epsilon$, the discretized action space $\Hc_\epsilon$ on $\Hc$, and the sampled data $\widehat{r}$ and $\widehat{\Phi}$ for both the dynamics and the rewards.

The key idea for the convergence analysis  is to decompose the error based on  these sources and to analyze each decomposed error accordingly. That is to consider the following different types of Bellman operators:
\begin{itemize}
    \item the operator $B$ in \eqref{equ:bellman_original} for  \eqref{mfc_objective_1};
    \item the operator $B_{\Hc_{\epsilon}}:\mathbb{R}^{\Cc}\to\mathbb{R}^{\Cc}$ which involves the discretized action space $\Hc_\epsilon$
    \begin{equation}\label{equ:bellman_approx_A}
        B_{\Hc_{\epsilon}}q(c)=r(c)+\gamma\max_{\tilde{h}\in\Hc_{\epsilon}}q(\Phi(c),\tilde{h});
    \end{equation}
    \item the operator $B_\epsilon$  in \reff{equ:bellman_kernel} defined on the $\epsilon$-net $\Cc_{\epsilon}$, 
which involves the discretized action space $\Hc_\epsilon$, and the kernel approximation;
    \item the operator $\widehat{B}_\epsilon: \mathbb{R}^{\Cc_\epsilon}\to\mathbb{R}^{\Cc_\epsilon}$  defined by
    \begin{equation}\label{equ:Bellman_biased_data}
        (\widehat{B}_\epsilon\,q)(c^i)=\widehat{r}(c^i)+\gamma\max_{\tilde{h}\in\Hc_{\epsilon}}\Gamma_Kq(\widehat{\Phi}(c^i),\tilde{h}),
    \end{equation}
    which involves the discretized action space $\Hc_\epsilon$, the kernel approximation, and the  estimated data.
    \item the operator $T$ that maps $\{f\in\R^{\Pc(\Xc)}:\|f\|_\infty\leq V_{\max}\}$ to itself, such that
    \begin{equation}\label{equ:Bellman_tildeV}
        Tv(\mu)=\max_{h\in\Hc_\epsilon}(r(\mu,h)+\gamma v(\Phi(\mu,h))).
    \end{equation}
\end{itemize}
We show that under mild assumptions, each of the above operators admits a unique fixed point.

\begin{Lemma}\label{lemma:Contra_Q}
    Assume Assumption~\ref{ass:r_MFC}. Let $V_{\max}:=\frac{R}{1-\gamma}$. Then,
    \begin{itemize}
    \setlength{\itemindent}{.5in}
        \item $B$ in \eqref{equ:bellman_original} has a unique fixed point in $\{f\in\R^{\Cc}:\|f\|_\infty\leq V_{\max}\}$. That is, there exists a unique $Q$ such that
        \begin{eqnarray}
         ({B}\,Q)(c)={r}(c)+\gamma\max_{\tilde{h}\in\Hc}Q({\Phi}(c),\tilde{h}).
        \end{eqnarray}
        \item $B_{\Hc_{\epsilon}}$ in \eqref{equ:bellman_approx_A} has a unique fixed point  in $\{f\in\R^{\Cc}:\|f\|_\infty\leq V_{\max}\}$. That is, there exists a unique
        $Q_{\Hc_\epsilon}$ such that
        \begin{eqnarray}\label{eq:tildeQc}
        B_{\Hc_{\epsilon}}Q_{\Hc_\epsilon}(c)=r(c)+\gamma\max_{\tilde{h}\in\Hc_{\epsilon}}Q_{\Hc_\epsilon}(\Phi(c),\tilde{h}).
    \end{eqnarray}
        \item $B_\epsilon$ in \eqref{equ:bellman_kernel} has a unique fixed point in $\{f\in\R^{\Cc_\epsilon}:\|f\|_\infty\leq V_{\max}\}$. That is, there exists a unique $Q_\epsilon$ such that for any $c^i\in\Cc_\epsilon$,
        \begin{eqnarray}\label{eq:Qc_epsilon}
         (B_\epsilon\,Q_\epsilon)(c^i)=r(c^i)+\gamma\max_{\tilde{h}\in\Hc_{\epsilon}}\Gamma_KQ_\epsilon(\Phi(c^i),\tilde{h}).
        \end{eqnarray}
        \item $\widehat{B}_\epsilon$ in \eqref{equ:Bellman_biased_data} has a unique fixed point in $\{f\in\R^{\Cc_\epsilon}:\|f\|_\infty\leq V_{\max}\}$. That is, there exists a unique $\widehat{Q}_\epsilon$ such that for any $c^i\in\Cc_\epsilon$, and $\widehat{r}, \widehat{\Phi}$ sampled from $c^i$'s $\epsilon$-neighborhood,
        \begin{eqnarray}\label{eq:hatQc_epsilon}
         (\widehat{B}_\epsilon\,\widehat{Q}_\epsilon)(c^i)=\widehat{r}(c^i)+\gamma\max_{\tilde{h}\in\Hc_{\epsilon}}\Gamma_K\widehat{Q}_\epsilon(\widehat{\Phi}(c^i),\tilde{h}).
        \end{eqnarray}
        \item $T$ has a unique fixed point $V_{\Hc_\epsilon}$ in $\{f\in\R^{\Pc(\Xc)}:\|f\|_\infty\leq V_{\max}\}$. That is
        \begin{equation}\label{eq:tildeV_c}
            T\,V_{\Hc_\epsilon}(\mu)=\max_{h\in\Hc_\epsilon}(r(\mu,h)+\gamma V_{\Hc_\epsilon}(\Phi(\mu,h))).
        \end{equation}
    \end{itemize}
\end{Lemma}

\begin{Lemma}[Characterization of $Q_{\Hc_\epsilon}$]\label{lemma:shenqi_Q}
    Assume Assumption~\ref{ass:r_MFC}. $V_{\Hc_\epsilon}$ in \eqref{eq:tildeV_c} is the optimal value function for the following MFC problem with continuous state space $\Pc(\Xc)$ and discretized action space $\Hc_\epsilon$. 
    \begin{equation}\label{mfc_objective_dis}
        V_{\Hc_\epsilon}(\mu) = \sup_{\pi\in\Pi_\epsilon}\sum_{t=0}^{\infty} \gamma^{t} r(\mu_t, \pi_t(\mu_t))
    \end{equation}
with $\Pi_\epsilon := \{\pi=\{\pi_t\}_{t=0}^\infty | \pi_t:\Pc(\Xc)\to\Hc_\epsilon\}$,
    \vspace{-0.1cm}
    subject to
    \vspace{-0.1cm}
    \begin{equation}\label{mfc_dynamics_dis}
        \mu_{t + 1}=\Phi(\mu_t, \pi_t(\mu_t)), \mu_0=\mu.
    \end{equation}
Moreover, $Q_{\Hc_\epsilon}$ in \eqref{eq:tildeQc} and $V_{\Hc_\epsilon}$ in \eqref{eq:tildeV_c} satisfy the following relation: 
    \begin{equation}\label{equ:shenqi_Q}
        Q_{\Hc_\epsilon}(\mu,h)=r(\mu,h)+\gamma V_{\Hc_\epsilon}(\Phi(\mu,h)),
    \end{equation}
    and $Q_{\Hc_\epsilon}$ is Lipschitz continuous.
    
\end{Lemma}
This connection between $Q_{\Hc_\epsilon}$ and the optimal value function $V_{\Hc_\epsilon}$ of the MFC problem with continuous state space $\Pc(\Xc)$ and discretized action space $\Hc_\epsilon$, is critical for estimating the error bounds in the convergence analysis.

\begin{Theorem}[Convergence]\label{thm:conv_mfc}
    Given $\epsilon > 0$. Assume Assumptions~\ref{ass:r_MFC},~\ref{ass:P_MFC},~\ref{ass:dynamic}, and ~\ref{ass:kernel}, and  $\gamma\cdot L_\Phi < 1$. Let $\widehat{B}_\epsilon: \mathbb{R}^{\Cc_\epsilon}\to\mathbb{R}^{\Cc_\epsilon}$ be the operator defined in \eqref{equ:Bellman_biased_data}
    \begin{eqnarray*}
        (\widehat{B}_\epsilon\,q)(c^i)=\widehat{r}(c^i)+\gamma\max_{\tilde{h}\in\Hc_{\epsilon}}\Gamma_Kq(\widehat{\Phi}(c^i),\tilde{h}),
    \end{eqnarray*}
    where $\widehat{r}(c)$ and $\widehat{\Phi}(c)$ are sampled from an $\epsilon$-neighborhood of $c$, then it has a unique fixed point $\widehat{Q}_\epsilon$ in $\{f\in\mathbb{R}^{\Cc_\epsilon}:||f||_\infty\leq V_{\max}\}$. Moreover, the sup distance between $\Gamma_K\widehat{Q}_\epsilon$ in \reff{kernel_def} and ${Q}$ in \reff{equmfcQ} is
    \begin{equation}\label{equ:error_bound}
        ||Q-\Gamma_K\widehat{Q}_\epsilon||_\infty\leq\frac{L_r + 2\gamma N_K L_K V_{\max} L_{\Phi}}{1-\gamma}\cdot\epsilon+\frac{2 L_r}{(1-\gamma L_{\Phi})(1-\gamma)}\cdot\epsilon.
    \end{equation}
In particular, for a fixed $\epsilon$, Algorithm~\ref{QL_CD} converges linearly to $\widehat{Q}_\epsilon$.
\end{Theorem}

\begin{proof}[Proof of Theorem \ref{thm:conv_mfc}]
The proof of the the convergence is to  quantify  $ ||Q-\Gamma_K\widehat{Q}_\epsilon||_\infty$ from the following estimate
\begin{equation}\label{equ:error_decomp}
    ||Q-\Gamma_K\widehat{Q}_\epsilon||_\infty\leq \underbrace{||Q-Q_{\Hc_\epsilon}||_\infty}_{(I)}+\underbrace{||Q_{\Hc_\epsilon}-\Gamma_K{Q}_{\epsilon}||_\infty}_{(II)}+\underbrace{||\Gamma_K{Q}_{\epsilon}-\Gamma_K\widehat{Q}_\epsilon||_\infty}_{(III)}.
\end{equation}
    
    (I) can be regarded as the approximation error from discretizing the lifted action space $\Hc$ by $\Hc_\epsilon$; (II) is the error from the kernel regression on $\Cc$ with the $\epsilon$-net $\Cc_\epsilon$; and (III) is estimating the error introduced by the sampled data $\widehat{r}$ and $\widehat{\Phi}$.

    {\it Step 1.} We shall use \ref{lemma:exist_policy} and Lemmas~\ref{lemma:shenqi_Q} to show that $||Q-Q_{\Hc_\epsilon}||_\infty \leq \frac{ L_r}{(1-\gamma L_{\Phi})(1-\gamma)}\cdot\epsilon$.
    By Lemma \ref{lemma:shenqi_Q}, $Q(c)-Q_{\Hc_\epsilon}(c)=\gamma\big( V(\Phi(c))-V_{\Hc_\epsilon}(\Phi(c))\big)$, where $V$ is the optimal value function of the problem on
    $\Pc(\Xc)$ and $\Hc$ in \eqref{mfc_objective_1}, and $V_{\Hc_\epsilon}$ is the optimal value function of the problem on $\Pc(\Xc)$ and $\Hc_\epsilon$ \eqref{mfc_objective_dis}-\eqref{mfc_dynamics_dis}. Hence it suffices to prove that $||{V}-V_{\Hc_\epsilon}||_\infty\leq\frac{L_r}{(1-\gamma L_{\Phi})(1-\gamma)}\cdot\epsilon$. We adopt the similar strategy as in the proof of Lemma \ref{lemma:exist_policy}.

    Let $\pi^*$ be the optimal policy of  \reff{mfc_objective_1}, whose existence is shown in Lemma~\ref{lemma:exist_policy}. For any $\mu\in\Pc(\Xc)$, let $(\mu,h)=(\mu_0,h_0),(\mu_1,h_1),(\mu_2,h_2),\dots,(\mu_t,h_t),\dots$ be the trajectory of the system under the optimal policy
    $\pi^*$, starting from $\mu$. We have $V(\mu)=\sum_{t=0}^\infty\gamma^tr(\mu_t,h_t)$.

    Now let $h^{i_t}$ be the nearest neighbor of $h_t$ in $\Hc_\epsilon$. $d_\Hc(h^{i_t}, h_t)\leq\epsilon$. Consider the trajectory of the system starting from $\mu$ and then taking $h^{i_0},\dots,h^{i_t},\dots$, denote the corresponding state by $\mu'_{t}$. We have
    $V_{\Hc_\epsilon}\geq\sum_{t=0}^\infty\gamma^t r(\mu'_t,h^{i_t})$, since $V_{\Hc_\epsilon}$ is the optimal value function.
    \begin{eqnarray*}
        d_{\Pc(\Xc)}(\mu'_{t}, \mu_t)=d_{\Pc(\Xc)}\big(\Phi(\mu'_{t-1}, h^{i_{t-1}}), \Phi(\mu_{t-1}, h_t)\big)\leq L_{\Phi}\cdot \big(d_{\Pc(\Xc)}(\mu'_{t-1}, \mu_{t-1})+\epsilon\big)
    \end{eqnarray*}
By the iteration,  we have  $d_{\Pc(\Xc)}(\mu'_{t}, \mu_t) \leq  \frac{L_\Phi - L_{\Phi}^{t+1}}{1- L_{\Phi}} \cdot \epsilon$, and
$|r(\mu'_t,h^{i_t})-r(\mu_t,h_t)|\leq L_r\cdot \big(d_{\Pc(\Xc)}(\mu'_{t}, \mu_t)+\epsilon\big)\leq L_r\cdot \frac{L_{\Phi}^{t+1}-1}{L_{\Phi}-1}\cdot \epsilon,$
    which implies
    \begin{eqnarray*}
        0\leq V(\mu)-V_{\Hc_\epsilon}(\mu)\leq \sum_{t=0}^\infty\gamma^t(r(\mu_{t},h_{t}) - r(\mu'_{t},h^{i_t}))
        \leq\sum_{t=0}^\infty\gamma^t\cdot L_r\cdot \frac{L_{\Phi}^{t+1}-1}{L_{\Phi}-1}\cdot \epsilon=\frac{L_r}{(1-\gamma L_{\Phi})(1-\gamma)}\cdot \epsilon.
    \end{eqnarray*}
    Here $0\leq V(\mu)-V_{\Hc_\epsilon}(\mu)$ is by the optimality of $V_{\Cc}$.

    {\it Step 2.} We shall use Lemmas~\ref{lemma:Contra_Q} and \ref{lemma:shenqi_Q} to show that $||Q_{\Hc_\epsilon}-\Gamma_K{Q}_{\epsilon}||_\infty \leq \frac{  L_r}{(1-\gamma L_{\Phi})(1-\gamma)}\cdot\epsilon$. Note that  
    \begin{eqnarray*}
        &&||\Gamma_KQ_\epsilon-Q_{\Hc_\epsilon}||_\infty=||\Gamma_KB_\epsilon Q_\epsilon-Q_{\Hc_\epsilon}||_\infty=||\Gamma_K B_{\Hc_{\epsilon}} \Gamma_K Q_\epsilon-Q_{\Hc_\epsilon}||_\infty\\
        &\leq& ||\Gamma_K B_{\Hc_{\epsilon}} \Gamma_K Q_\epsilon - \Gamma_K B_{\Hc_{\epsilon}} Q_{\Hc_\epsilon}||_\infty + ||\Gamma_K B_{\Hc_{\epsilon}} Q_{\Hc_\epsilon} - Q_{\Hc_\epsilon}||_\infty\\
        &=& ||\Gamma_K B_{\Hc_{\epsilon}} \Gamma_K Q_\epsilon - \Gamma_K B_{\Hc_{\epsilon}} Q_{\Hc_\epsilon}||_\infty + ||\Gamma_K Q_{\Hc_\epsilon} - Q_{\Hc_\epsilon}||_\infty
        \leq \gamma ||\Gamma_KQ_\epsilon-Q_{\Hc_\epsilon}||_\infty + ||\Gamma_K Q_{\Hc_\epsilon} - Q_{\Hc_\epsilon}||_\infty.
    \end{eqnarray*}
    Here the first and the third equalities hold since $Q_\epsilon$ is the fixed point of $B_\epsilon$ and $Q_{\Hc_\epsilon}$ is the fixed point of $B_{\Hc_{\epsilon}}$.
    The second inequality is by the fact that $\Gamma_K$ is a non-expansion mapping, i.e., $\|\Gamma_K f\|_\infty\leq\|f\|_\infty$, and that $B_{\Hc_{\epsilon}}$ is a contraction with modulus $\gamma$ with the supremum norm.
    Meanwhile, for any Lipschitz function $f\in\R^\Cc$ with Lipschitz constant $L$, we have for all $c\in\Cc$,
    \begin{eqnarray*}
        |\Gamma_K f(c) - f(c)|=\sum_{i=1}^{N_\epsilon}K(c,c^i)|f(c^i)-f(c)|\leq\sum_{i=1}^{N_\epsilon}K(c,c^i)\epsilon L=\epsilon L.
    \end{eqnarray*}
Note here the inequality follows from $K(c,c^i)=0$ for all $d_{\Cc}(c,c^i)\geq\epsilon$.
    Therefore, 
$        ||\Gamma_KQ_\epsilon-Q_{\Hc_\epsilon}||_\infty \leq \frac{L_{Q_{\Hc_\epsilon}}}{1-\gamma}\epsilon$,
    where $L_{Q_{\Hc_\epsilon}}=\frac{L_r}{1-\gamma L_{\Phi}}$ is the Lipschitz constant for $Q_{\Hc_\epsilon}$.
    
    {\it Final step.} 
    Let $q_0$ denote the zero function on $\Cc_\epsilon$. By Lemma \ref{lemma:Contra_Q}, $Q_\epsilon=\lim_{n\to\infty}B_\epsilon^n q_0$, and $\widehat{Q}_\epsilon=\lim_{n\to\infty}\widehat{B}_\epsilon^n q_0$. Denote $q_n:=B_\epsilon^n q_0$, $\widehat{q}_n:=\widehat{B}_\epsilon^n q_0$, and $e_n:=||q_n-\widehat{q}_n||_{\infty}$. For any $c\in \Cc_\epsilon$,
    \begin{eqnarray*}
         e_{n+1}(c)&=&\big| \widehat{r}(c)+\gamma\max_{\tilde{h}\in\Hc_{\epsilon}}\Gamma_K\widehat{q}_n(\widehat{\Phi}(c),\tilde{h})-r(c)-\gamma\max_{\tilde{h}\in\Hc_{\epsilon}}\Gamma_Kq_n(\Phi(c),\tilde{h})\big|\\
        &\leq&|\widehat{r}(c)-r(c)|+\gamma\max_{\tilde{h}\in\Hc_{\epsilon}}\big|\Gamma_K\widehat{q}_n(\widehat{\Phi}(c),\tilde{h})-\Gamma_Kq_n(\Phi(c),\tilde{h})\big|\\
        &\leq&\epsilon L_r+\gamma\max_{\tilde{h}\in\Hc_{\epsilon}}\big[|\Gamma_K\widehat{q}_n(\widehat{\Phi}(c),\tilde{h})-\Gamma_K\widehat{q}_n(\Phi(c),\tilde{h})|+|\Gamma_K\widehat{q}_n(\Phi(c),\tilde{h})-\Gamma_Kq_n(\Phi(c),\tilde{h})|\big].
    \end{eqnarray*}
    Here $|\widehat{r}(c) - r(c)|\leq\epsilon L_r$ because $\widehat{r}(c)$ is sampled from an $\epsilon$-neighborhood of $c$ and by Assumption~\ref{ass:r_MFC}. Moreover, for any fixed $\tilde{h}$, 
    \begin{eqnarray*}
        |\Gamma_K\widehat{q}_n(\widehat{\Phi}(c),\tilde{h})-\Gamma_K\widehat{q}_n(\Phi(c),\tilde{h})| &=& |\sum_{i=1}^{N_\epsilon}(K(c^i,(\widehat{\Phi}(c),\tilde{h}))-K(c^i,(\Phi(c),\tilde{h})))\widehat{q}_n(c^i)|\\
        &\leq& 2N_K L_K V_{\max}\cdot d_{\Pc(\Xc)}(\widehat{\Phi}(c),\Phi(c))\leq 2N_K L_K V_{\max} L_{\Phi} \epsilon.
    \end{eqnarray*}
    The first inequality comes from Assumption \ref{ass:kernel}, because $K(c^i,(\widehat{\Phi}(c),\tilde{h}))-K(c^i,(\Phi(c),\tilde{h}))$ is nonzero for at most $2N_K$ index $i\in\{1,2,\dots,N_\epsilon\}$, $K$ is Lipschitz continuous, and $||\widehat{q}_n||_\infty\leq V_{\max}$.
    The second inequality comes from the fact that $\widehat{\Phi}(c)$ is sampled from an $\epsilon$-neighborhood of $c$ and by Assumption~\ref{ass:P_MFC}. Meanwhile,
    \begin{eqnarray*}
        |\Gamma_K\widehat{q}_n(\Phi(c),\tilde{h})-\Gamma_Kq_n(\Phi(c),\tilde{h})|\leq||q_n-\widehat{q}_n||_{\infty}=e_n,
    \end{eqnarray*}
    since $\Gamma$ is non-expansion. Putting these pieces together, we have 
    \begin{eqnarray*}
        e_{n+1}=\max_{c\in\Cc_\epsilon}e_{n+1}(c)\leq\epsilon L_r + \epsilon\gamma 2N_K L_K V_{\max} L_{\Phi} + \gamma e_n.
    \end{eqnarray*}
    In this case, elementary algebra shows that $e_n\leq\epsilon\cdot\frac{L_r + \gamma 2N_K L_K V_{\max} L_{\Phi}}{1-\gamma}, \forall n$. Then since $\Gamma_K$ is non-expansion, $||\Gamma_K{Q}_{\Cc_\epsilon}-\Gamma_K\widehat{Q}_\epsilon||_\infty\leq\epsilon\cdot\frac{L_r + \gamma 2N_K L_K V_{\max} L_{\Phi}}{1-\gamma}$, hence the error bound \eqref{equ:error_bound}.
    
    The claim regarding the convergence rate follows from the $\gamma-$contraction of operator $\widehat{B}_\epsilon$.
\end{proof}

\subsection{Sample Complexity Analysis} \label{subsec:MFC_complexity}
In classical Q-learning for MDPs with stochastic environment,  every component in the $\epsilon$-net is required to be visited a number of times in order to get desirable estimate for the Q function. The usual terminology {\it covering time} refers to the expected {number of steps}  to visit every component in the $\epsilon$-net at least once, for a given exploration policy. The complexity analysis thus  focuses on the necessary rounds of the covering time. 

In contrast, visiting each component in the $\epsilon$-net {\it once} is sufficient with deterministic dynamics. We will demonstrate  that using deterministic mean-field dynamics to approximate N-agent stochastic environment will indeed significantly reduce the complexity analysis.

To start, denote $T_{\Cc, \pi}$ as the covering time of the $\epsilon$-net under (random) policy $\pi$, such that
\begin{eqnarray*}\label{def:T_c}
    T_{\Cc, \pi}&:=&\sup_{\mu\in\Pc(\Xc)}\inf\Big\{t>0:\mu_0=\mu,\forall c^i\in\Cc_\epsilon, \exists t_i\leq t,\\
    &&(\mu_{t_i}, h_{t_i})\text{ in the } \epsilon\text{-neighborhood of } c^i\text{, under the policy }\pi\Big\}.
\end{eqnarray*}
Recall that an $\epsilon'$-greedy policy
on $\Hc_\epsilon$ is a policy which with probability at least $\epsilon'$ will uniformly explore the actions on $\Hc_\epsilon$. Note that this type of policy always exists. 
And we have the following sample complexity result.

\begin{Theorem}[Sample complexity]\label{thm:sample_complexity}
    Given $\epsilon, \delta > 0$ and Assumption~\ref{ass:dynamic}, for any $\epsilon'>0$, let $\pi_{\epsilon'}$ be an $\epsilon'$-greedy policy
    on $\Hc_\epsilon$.
    Then
    \begin{equation}\label{equ:bound_T_c}
        \E[T_{\Cc,\pi_{\epsilon'}}]\leq\frac{(M_\epsilon+1) \cdot (N_{\Hc_\epsilon})^{M_\epsilon+1}}{(\epsilon')^{M_\epsilon+1}}\cdot \log(N_\epsilon).
    \end{equation}
    Here $M_\epsilon$ is defined in Assumption \ref{ass:dynamic}. Moreover, with probability $1-\delta$, for any initial state $\mu$, under the $\epsilon'$-greedy policy, the dynamics will visit
    each $\epsilon$-neighborhood of elements in $\Cc_\epsilon$ at least once, after
    \begin{equation}\label{equ:samp_comp}
        \frac{(M_\epsilon+1) \cdot (N_{\Hc_\epsilon})^{M_\epsilon+1}}{(\epsilon')^{M_\epsilon+1}}\cdot \log(N_\epsilon) \cdot e \cdot \log(1/\delta).
    \end{equation}
    time steps, where $\log(N_\epsilon)=\Theta(|\Xc||\Uc|\log(1/\epsilon))$, and $N_{\Hc_\epsilon}=\Theta((\frac{1}{\epsilon})^{(|\Uc|-1|)|\Xc|})$.
\end{Theorem}

Theorem~\ref{thm:sample_complexity} provides an upper bound $\Omega(\text{poly}((1/\epsilon)\cdot\log(1/\delta)))$ for the covering time under the $\epsilon'$-greedy policy,
 in terms of the size of the $\epsilon$-net and the accuracy $1/\delta$. 
The proof of Theorem \ref{thm:sample_complexity} relies on the following lemma.

\begin{Lemma}\label{lemma:covering}
    Assume for some policy $\pi$, $\E[T_{\Cc,\pi}]\leq T<\infty$. Then with probability $1-\delta$, for any initial state $\mu$, under the policy $\pi$, the dynamics will visit
    each $\epsilon$-neighborhood of elements in $\Cc_\epsilon$ at least once, after $T \cdot e \cdot \log(1/\delta)$ time steps, i.e. $\P(T_{\Cc,\pi}\leq T \cdot e \cdot \log(1/\delta))\geq 1-\delta$.
\end{Lemma}

\begin{proof}[Proof of Theorem~\ref{thm:sample_complexity}]
    Recall there are $N_\epsilon$ different pairs in the $\epsilon$-net. Denote the $\epsilon$-neighborhoods of those pairs by
    $B_\epsilon=\{B^i\}_{i=1}^{N_\epsilon}$. Without loss of generality, we may assume that $B^i$ are disjoint, since the covering time will only become smaller if they overlap with each other. 
    Let $T_k:=\min\{t>1:\text{k of }B_\epsilon\text{ is visited}\}$. $T_k-T_{k-1}$ is the time to 
    visit a new neighborhood after $k-1$ neighborhoods are visited. By Assumption~\ref{ass:dynamic}, for any $B^i\in B_\epsilon$ with center $(\mu^i, h^i)$, $\mu\in\Pc(\Xc)$, there
    exists a sequence of actions in $\Hc_\epsilon$, whose length is at most $M_\epsilon$, such that starting from $\mu$ and taking that sequence of actions
    will lead the  visit of the $\epsilon$-neighborhood of $\mu^i$. Then, at that point, taking $h^i$ will yield the  visit of $B^i$. Hence 
    $\forall B^i\in B_\epsilon$, $\mu\in\Pc(\Xc)$,
    \begin{eqnarray*}
        &&\P(B^i\text{ is visited in } M_\epsilon+1 \text{ steps}\,|\, \mu_{T_{k-1}}=\mu)\geq \left(\frac{\epsilon'}{N_{\Hc_\epsilon}}\right)^{M_\epsilon+1}.\\
        &&\P(\text{a new neighborhood is visited in } M_\epsilon+1 \text{ steps}\, | \mu_{T_{k-1}}=\mu)\geq(N_\epsilon-k+1)\cdot \left(\frac{\epsilon'}{N_{\Hc_\epsilon}}\right)^{M_\epsilon+1}.
    \end{eqnarray*}
This implies $\E[T_k-T_{k-1}]\leq \frac{M_\epsilon+1}{N_\epsilon-k+1}\cdot (\frac{N_{\Hc_\epsilon}}{\epsilon'})^{M_\epsilon+1}$. Summing $\E[T_k-T_{k-1}]$ from $k=1$ to $k=N_\epsilon$ yields the desired result. The second part follows directly from Lemma \ref{lemma:covering}. Meanwhile, $N_{\Hc_\epsilon}$, the size of the $\epsilon$-net in $\Hc$ is $\Theta((\frac{1}{\epsilon})^{(|\Uc|-1)|\Xc|})$, because $\Hc$ is a compact $(|\Uc|-1)|\Xc|$ dimensional manifold. Similarly, $N_{\epsilon}=\Theta((\frac{1}{\epsilon})^{|\Uc||\Xc|-1})$ as $\Cc$ is a compact $|\Uc||\Xc|-1$ dimensional manifold.
\end{proof}

\section{Mean-field Approximation to Cooperative MARL}\label{sec:connection}
In this section, we provide a complete description of the connections between cooperative MARL and MFC, in terms of the value function approximation and algorithmic approximation under the context of learning.
\subsection{Value Function Approximation}\label{sec:value_approx}
First we will show that under the Pareto optimality criterion,
  \reff{mfc_objective_2}  is an approximation to its corresponding  cooperative MARL, with an error of $\mathcal{O}(\frac{1}{\sqrt{N}})$. 

Recall the admissible policy $\pi = \{\pi_t\}_{t=0}^\infty \in \Pi$. Note that the cooperative MARL in Section \ref{subsec:nplayer} with $N$ identical, indistinguishable, and interchangeable agents becomes
\begin{linenomath}
\begin{align} \label{rewardagentj}
& \sup_{\pi} u_N^{\pi}(\mu^N):=\sup_{\pi} \frac{1}{N}\sum_{j=1}^N v^{j, \pi} (x^{j, N}, \mu^N) \nonumber =  \sup_{\pi} \frac{1}{N} \sum_{j=1}^N\E\Big[\sum_{t=0}^\infty \gamma^t \tilde r(x_t^{j, N}, \mu_t^N, u_t^{j, N}, \nu_t^N)\Big], \tag{MARL}\\
 &\text{subject to} \;\;  x_{t + 1}^{j, N} \sim P(x_t^{j, N}, \mu_t^N, u_t^{j, N}, \nu_t^N), \; u_t^{j, N} \sim \pi_t(x_t^{j, N}, \mu_t^N), \;\;\; 1 \leq j \leq N, \nonumber
\end{align}
\end{linenomath}
with initial conditions $x_0^{j,N} = x^{j, N}$ ($j=1,2,\cdots,N$) and $\mu_0^N(x) =\mu^N(x) := \frac{\sum_{j=1}^N 1(x^{j, N} = x)}{N}$ for $x\in \mathcal{X}$. By symmetry, one can denote $u_N^{\pi}(\mu^N):=\frac{1}{N}\sum_{j=1}^N v^{j, \pi} ({x^{j, N}, \mu^N})$.
\begin{Definition} $\pi^\epsilon$ is $\epsilon$-Pareto optimal for  \reff{rewardagentj} if
\begin{eqnarray*}
 u_N^{ \pi^\epsilon} \geq \sup_{\pi}  u_N^{ \pi} - \epsilon.
\end{eqnarray*}
\end{Definition}

\begin{Assumption}[Continuity of $\pi$]\label{ass: pi}
There exists $L_{\Pi} > 0$ such that for all $x \in \Xc, \mu_1, \mu_2 \in \Pc(\Xc)$, and $\pi \in \Pi$,
\begin{eqnarray*}
    \|\pi_t(\mu_1, x) - \pi_t(\mu_2, x)\|_1 \leq L_{\Pi} \|\mu_1 - \mu_2\|_1, \;\;\;\; \mbox{for any}\; t \geq 0.
\end{eqnarray*}
\end{Assumption}
This Lipschitz assumption for admissible policies is commonly used to bridge games in the  N-player setting and the mean-field setting \cite{HMC2006,GHXZ2019}.

We are now ready to show that the optimal policy for \reff{mfc_objective_2}   is approximately Pareto optimal for  \reff{rewardagentj} when $N \to \infty$.

\begin{Theorem}[Approximation]
    \label{NagentMFC} Assume $\gamma\cdot(2L_{P}+1)(1  + L_{\Pi}) < 1$ and
Assumptions \ref{ass:r_MFC}, \ref{ass:P_MFC} and \ref{ass: pi}, 
then there exists constant $C = C(L_P, L_{\tilde r}, L_{\Pi}, {{|\Xc|, |\Uc|}}, \tilde R,  \gamma)$, {depending on the dimensions of the state and action spaces in a sublinear order $(\sqrt{|\mathcal{X}|}+\sqrt{\mathcal|\Uc|})$}, and independent of the number of agents $N$, 
 such that
\beq \label{equvalueappro}
\sup_{\pi} \Big|u_N^{\pi}(\mu^N) - v^\pi (\mu^N)\Big| \leq C\frac{1}{\sqrt{N}},
\enq
for any initial condition $x_0^{j, N}=x^j$ $(j=1,2,\cdots,N)$ and $\mu^N(x) =\frac{\sum_{j=1}^N 1({x^{j, N}} = x)}{N}$ $(x\in \mathcal{X})$. Here $v^\pi$ and $u_N^{\pi}$ are given in \eqref{mfc_objective_2} and \eqref{rewardagentj} respectively.  Consequently, for any $\epsilon_1 > 0$, there exists an integer $D_{\epsilon_1}\in \mathbb{N}$ such that when $N \geq D_{\epsilon_1}$, any $\epsilon_2$-optimal policy for  \reff{mfc_objective_2} with learning  is $(\epsilon_1 + \epsilon_2)$-Pareto optimal for \reff{rewardagentj} with $N$ players.
\end{Theorem}

\begin{corollary}[Optimal value approximation]\label{corr:optimal_value_app} Assume the same conditions as in Theorem \ref{NagentMFC}. Further assume that there exists an optimal policy satisfying Assumption \ref{ass: pi} for \reff{mfc_objective_2} and \reff{rewardagentj}. Denote $\pi^* \in \arg\sup_{\pi \in \Pi} v^{\pi}$ and $\widetilde{\pi} \in \arg\sup_{\pi \in \Pi} u_N^{\pi}$, there exists a constant $C = C(L_P, L_{\tilde r}, L_{\Pi}, { {|\Xc|, |\Uc|}}, \tilde R, \gamma)$, {depending on the dimensions of the state and action spaces 
in a sublinear order $(\sqrt{|\mathcal{X}|}+\sqrt{\mathcal|\Uc|})$}, such that
\begin{eqnarray}
\left|v^{\pi^*}(\mu^N)- u^{\widetilde{\pi}}(\mu^N)\right|\leq \frac{C}{\sqrt{N}},
\end{eqnarray}
with initial conditions $x_0^{j,N} = x^{j, N}$ and $\mu^N := \frac{\sum_{j=1}^N 1(x^{j, N} = x)}{N}$.
\end{corollary}
Corollary \ref{corr:optimal_value_app} follows directly from Theorem \ref{NagentMFC} and the proof is deferred to Appendix \ref{app:proof}.


\begin{proof}[Proof of Theorem \ref{NagentMFC}]
First, by \reff{equ: r}
\begin{eqnarray*}
u_N^{\pi}(\mu^N)
&=&\frac{1}{N} \sum_{j=1}^N \sum_{t=0} \gamma^t \E\big[\tilde r(x_t^{j, N}, \mu_t^N, u_t^{j, N}, \nu_t^N)\big]  - \frac{1}{N} \sum_{j=1}^N \sum_{t=0} \gamma^t \E\big[\tilde r(x_t^{j, N}, \mu_t^N, u_t^{j, N}, \tilde\nu_t^N)\big] \\
& & \;\;+\;  \sum_{t=0}^\infty \gamma^t \E\big[r(\mu_t^N, \pi_t(\mu_t^N))\big],\\
v^\pi(\mu^N) &=& \sum_{t=0}^\infty \gamma^t \E[\tilde r(x_t, \mu_t, u_t, \nu_t)] = \sum_{t=0}^\infty \gamma^t r(\mu_t, \pi_t(\mu_t)),
\end{eqnarray*}
where  $\tilde \nu_t^N(u): = \sum_{x \in \Xc} \pi_t(\mu, x)(u)\mu_t^N(x)  = \frac{1}{N} \sum_{j=1}^N \pi_t(\mu_t^N, x_t^{j,N})(u)$.

By the continuity of $r$ from Lemma \ref{lemma:cont_r} and Assumption \ref{ass:r_MFC}, 
\begin{eqnarray*}
& & \sup_{\pi} \Big|u_N^\pi(\mu^N) - v^\pi(\mu^N)\Big| \\
&\leq& (\tilde R + 2 L_{\tilde r})\sum_{t=0}^\infty \gamma^t \sup_{\pi} \Big( \E \Big[\|\mu_t^{N, \pi} - \mu_t^\pi\|_1  \Big] + \E\Big[\|\pi_t(\mu_t) - \pi_t(\mu_t^N)\|_1\Big]\Big)\\
& & + L_{\tilde r} \sum_{t=0}^\infty \gamma^t \sup_{\pi}\E\Big[\|\nu_t^{N, \pi} - \tilde \nu_t^{N, \pi}\|_1\Big]\\
 & \leq & (\tilde R + 2 L_{\tilde r}) (1 + L_{\Pi}) \sum_{t=0}^\infty \gamma^t \sup_{\pi} \E \Big[\|\mu_t^{N, \pi} - \mu_t^\pi\|_1  \Big] + L_{\tilde r} \sum_{t=0}^\infty \gamma^t \sup_{\pi}\E\Big[\|\nu_t^{N, \pi} - \tilde \nu_t^{N, \pi}\|_1\Big].
\end{eqnarray*}
To prove \reff{equvalueappro}, it is sufficient to estimate $\delta_t^{1, N}:=\sup_{\pi}\E[\|\mu_t^{N, \pi} - \mu_t^\pi\|_1]$ and $\delta_t^{2, N}:=\sup_{\pi} \E[\|\nu_t^{N, \pi} - \tilde \nu_t^{N, \pi}\|_1$.
First, we show that $\delta_t^{2, N} = \mathcal{O}(\frac{1}{\sqrt{N}})$. Denote for any $\nu \in \Pc(\Uc)$ and $f: \Uc \to \R$, $\nu(f): = \sum_{u \in \Uc} f(u)\nu(u)$. Then for any $t \geq 0$
{ 

\begin{eqnarray} \label{deltanu}
& & \E\left[\left\|\tilde\nu_t^N - \nu_t^N\right\|_1\right] = \E \left[\E\left[\left\|\tilde\nu_{t}^N - \nu_t^N\right\|_1 \bigg| x_t^{1, N}, \cdots, x_t^{N, N}\right]\right]\\
&=& \E\left[ \E\left[ \sup_{f: \Uc \to \{-1, 1\}} \big(\tilde\nu_t^N(f) - \nu_t^N(f)\big) \bigg| x_t^{1, N}, \cdots, x_t^{N, N} \right]\right] \nonumber\\
&=& \E\left[\E\left[\sup_{f: \Uc \to \{-1, 1\}} \frac{1}{N}\sum_{j=1}^N \sum_{u \in \Uc} \pi_t(\mu_t^N, x_t^{j,N})(u)f(u) - \frac{1}{N}\sum_{j=1}^N f(u_t^{j,N}) \bigg| x_t^{1, N}, \cdots, x_t^{N, N}\right]\right], \nonumber
\end{eqnarray} 
where the first equality is by law of total expectation and the last equality is by the definitions of $\tilde\nu_{t}^N$ and $\nu_t^N$.
Now consider a fixed $f: \Uc \to \{-1, 1\}$. Conditioned on $x_t^{1, N}, \cdots, x_t^{N, N}$,  $\{u_t^{j,N}\}_{j=1}^N$ is a sequence of independent random variables with $u_t^{j,N}\sim\pi_t(\mu_t^N, x_t^{j,N})(\cdot)$. Therefore, conditioned on $x_t^{1, N}, \cdots, x_t^{N, N}$, $\left\{\sum_{u \in \Uc} \pi_t(\mu_t^N, x_t^{j,N})(u)f(u) - f(u_t^{j,N})\right\}_{j=1}^N$ is a sequence of independent mean-zero random variables bounded in $[-2, 2]$. The boundedness further implies that each $\sum_{u \in \Uc} \pi_t(\mu_t^N, x_t^{j,N})(u)f(u) - f(u_t^{j,N})$ is a sub-Gaussian random variable with variance bounded by $4$. (See Chapter 2 of \cite{wainwright2019high} for the general introduction to sub-Gaussian random variables.)
Meanwhile, the independence implies that conditioned on $x_t^{1, N}, \cdots, x_t^{N, N}$,
\begin{equation*}
    \frac{1}{N}\sum_{j=1}^N \sum_{u \in \Uc} \pi_t(\mu_t^N, x_t^{j,N})(u)f(u) - \frac{1}{N}\sum_{j=1}^N f(u_t^{j,N})
\end{equation*}
is a mean-zero sub-Gaussian random variable with variance $\frac 4 N$. In general, for a sequence of mean-zero sub-Gaussian random variables $\{X_i\}_{i=1}^M$ with parameter $\sigma^2$, by Eqn.(2.66) in \cite{wainwright2019high}, we have
\begin{equation*}
    \E\left[\sup_{i=1,\cdots,M}X_i\right]\leq\sqrt{2\sigma^2\ln(M)}.
\end{equation*}
Therefore, conditioned on $x_t^{1, N}, \cdots, x_t^{N, N}$, 
\begin{equation*}
    \E\left[\sup_{f: \Uc \to \{-1, 1\}} \frac{1}{N}\sum_{j=1}^N \sum_{u \in \Uc} \pi_t(\mu_t^N, x_t^{j,N})(u)f(u) - \frac{1}{N}\sum_{j=1}^N f(u_t^{j,N}) \bigg| x_t^{1, N}, \cdots, x_t^{N, N}\right]\leq\sqrt{8\ln(2)|\Uc|/N}
\end{equation*}
holds since we have in total $2^{|\Uc|}$ different choices for $f: \Uc \to \{-1, 1\}$ when taking the supremum. Thus, following \eqref{deltanu}, we have  
\begin{eqnarray}\label{eq:bound}
\delta_t^{2, N}= \sup_{\pi}\E\left[\left\|\tilde\nu_t^N - \nu_t^N\right\|_1\right]\leq\sqrt{8\ln(2)|\Uc|/N}.  
\end{eqnarray}}

\newpage

Second, we estimate $\delta_t^{1,N}$ and claim that $\delta_t^{1,N} = \mathcal{O}(\frac{1}{\sqrt{N}})$. This is done by induction.
The claim holds for $t=0$ because $\delta_0^{1,N} = 0$. Suppose the claim holds for $t$ and consider $t + 1$. 


{ Given $x_t^{1, N}, \cdots, x_t^{N, N}$, $\mu_t^N=\frac{1}{N}\sum_{j=1}^N \delta_{x_t^{j, N}}$ and policy $\pi_t(\mu_t^N)$ at time $t$, for any $\nu\in\Pc(\Uc)$, let $\mu_t^N P_{\mu_t^N, \nu}$ denote a $\Pc(\Xc)$-valued random variable, with
\begin{equation*}
    \mu_t^N P_{\mu_t^N, \nu}(x):=\frac{1}{N}\sum_{j=1}^N P(x_t^{j,N}, \mu^N_t, u_t^{j,N}, \nu)(x), \quad u_t^{j,N}\sim\pi_t(\mu_t^N, x_t^{j,N}).
\end{equation*}
}
{ 
We consider the following decomposition,

\beq \label{equmuNmuerror}
\E\left[\|\mu_{t + 1}^N - \mu_{t + 1}\|_1\right]
&\leq& \underbrace{\E\left[\left\|\mu_{t + 1}^N - \mu_t^N P_{\mu_t^N, \nu_t^N}\right\|_1\right]}_{(I)} + \underbrace{\E\left[\left\|\mu_t^N P_{\mu_t^N, \nu_t^N} - \mu_t^N P_{\mu_t^N, \tilde\nu_t^N}\right\|_1\right]}_{(II)} \nonumber\\
& & \qquad\;\;\; +\;
\underbrace{\E\left[\left\|\mu_t^N P_{\mu_t^N, \tilde\nu_t^N}-\Phi(\mu_t^N, \pi_t(\mu_t^N))\right\|_1\right]}_{(III)} +
\underbrace{\E\left[\left\|\Phi(\mu_t^N, \pi_t(\mu_t^N)) - \mu_{t + 1}\right\|_1\right]}_{(IV)}.
\enq

Bounding (I) in RHS of \reff{equmuNmuerror}: We proceed the similar argument as \reff{deltanu},
\begin{eqnarray*}
& & \E\left[\left\|\mu_{t + 1}^N - \mu_t^N P_{\mu_t^N, \nu_t^N}\right\|_1\right] =  \E \left[\E\left[\left\|\mu_{t + 1}^N - \mu_t^N P_{\mu_t^N, \nu_t^N}\right\|_1 \left| x_t^{1, N}, \cdots, x_t^{N, N},u_t^{1, N}, \cdots, u_t^{N, N}\right]\right]\right.\\
&=& \E\left[\E\left[\sup_{f: \Xc \to \{-1, 1\}} \frac{1}{N}\sum_{j=1}^N f(x_{t+1}^{j,N})-\frac{1}{N}\sum_{j=1}^N \sum_{x \in \Xc}P(x_t^{j,N}, \mu^N_t, u_t^{j,N}, \nu^N_t)(x)f(x)  \bigg| x_t^{1, N}, \cdots, x_t^{N, N}, u_t^{1, N}, \cdots, u_t^{N, N}\right]\right]\\
&\leq& \sqrt{8\ln(2)|\Xc|/N}.
\end{eqnarray*}

Bounding (II) in RHS of \reff{equmuNmuerror}:
\begin{eqnarray*}
& & \E\left[\left\|\mu_t^N P_{\mu_t^N, {\nu}_t^N} - \mu_t^N P_{\mu_t^N, \tilde{\nu}_t^N}\right\|_1\right]
= \E\left[\left\| \frac{1}{N}\sum_{j=1}^N P(x_t^{j,N}, \mu^N_t, u_t^{j,N}, {\nu}^N_t) - \frac{1}{N}\sum_{j=1}^N  P(x_t^{j,N}, \mu^N_t, u_t^{j,N}, \tilde{\nu}^N_t)\right\|_1\right]\\
&\leq&\frac{1}{N}\sum_{j=1}^N \E\left[\left\|  P(x_t^{j,N}, \mu^N_t, u_t^{j,N}, {\nu}^N_t) -   P(x_t^{j,N}, \mu^N_t, u_t^{j,N}, \tilde{\nu}^N_t)\right\|_1\right] \\
&\leq& L_P\cdot \mathbb{E}\left[\left\| {\nu}^N_t- \tilde{\nu}^N_t\right\|_1\right] \leq L_P\sqrt{8\ln(2)|\Uc|/N},
\end{eqnarray*}
in which the second last inequality holds by the Lipschitz property from Assumption \ref{ass:P_MFC} and the last inequality holds by \eqref{eq:bound}.

Bounding (III) in RHS of \reff{equmuNmuerror}:
\begin{eqnarray*}
& & \E\left[\left\| \mu_t^N P_{\mu_t^N, \tilde\nu_t^N} - \Phi(\mu_t^N, \pi_t(\mu^N_t))\right\|_1\right]\\
&=& \E\biggl[\E\biggl[\sup_{g: \Xc \to \{-1, 1\}} \frac{1}{N}\sum_{j=1}^N \sum_{x \in \Xc} P(x_t^{j,N}, \mu_t^N, u_t^{j,N}, \tilde\nu_t^N)(x) g(x) \\
&  & -  \frac{1}{N}\sum_{j=1}^N \sum_{u \in \Uc} \sum_{x \in \Xc} P(x_t^{j,N}, \mu_t^N, u, \tilde\nu_t^N)(x)\pi_t(\mu_t^N, x_t^{j,N})(u) g(x)\biggl|x_t^{1, N}, \ldots, x_t^{N, N}\biggl]\biggl]
\end{eqnarray*}
For a fixed $g: \Xc \to \{-1, 1\}$, conditioned on $x_t^{1,N},\cdots,x_t^{N,N}$,
\begin{equation*}
    \left\{\sum_{x \in \Xc} P(x_t^{j,N}, \mu_t^N, u_t^{j,N}, \tilde\nu_t^N)(x) g(x) - \sum_{u \in \Uc} \sum_{x \in \Xc} P(x_t^{j,N}, \mu_t^N, u, \tilde\nu_t^N)(x)\pi_t(\mu_t^N, x_t^{j,N})(u) g(x)\right\}_{j=1}^N
\end{equation*}
are independent mean-zero sub-Gaussian random variables. Meanwhile, since by definition, we have for each $j=1,\cdots,N$, $\sum_{x \in \Xc} P(x_t^{j,N}, \mu_t^N, u_t^{j,N}, \tilde\nu_t^N)(x)=1$, it is easy to show that $\sum_{x \in \Xc} P(x_t^{j,N}, \mu_t^N, u_t^{j,N}, \tilde\nu_t^N)(x) g(x)$ is bounded by $[-1, 1]$. Therefore, using the same argument  applied in the proof of \eqref{eq:bound}, we can show that
\begin{equation*}
    \E\left[\left\| \mu_t^N P_{\mu_t^N, \tilde\nu_t^N} - \Phi(\mu_t^N, \pi_t(\mu^N_t))\right\|_1\right]\leq\sqrt{8\ln(2)|\Xc|/N}.
\end{equation*}
}



Bounding (IV) in RHS of \reff{equmuNmuerror}:
\begin{eqnarray*}
\E\big[\|\Phi(\mu_t^N, \pi_t(\mu_t^N)) - \mu_{t + 1}\|_1\big]= \E\big[\|\Phi(\mu_t^{N}, \pi_t(\mu_t^N)) - \Phi(\mu_t, \pi_t(\mu_t))\|_1\big]
\leq (2 L_P + 1) (1 + L_{\Pi}) \E\big[\|\mu_t^N - \mu_t\|_1\big],
\end{eqnarray*}
where the first equality is from the flow of probability measure $\mu_{t + 1} = \Phi(\mu_t, \pi_t(\mu_t))$ by Lemma \ref{lemma:flowmut}, and the first inequality is by the continuity of $\Phi$ from Lemma \ref{lemma:cont_phi}.\\  

By taking supremum over $\pi$ on both sides of \reff{equmuNmuerror}, we have $\delta_{t+1}^{1,N} \leq (2L_P + 1)(1 +L_{\Pi}) \delta_t^{1, N} + { (L_P\sqrt{|\Uc|} + 2\sqrt{|\Xc|})\sqrt{8\ln(2)/N}}$, hence $\delta_t^{1, N} \leq \frac{{ (L_P\sqrt{|\Uc|} + 2\sqrt{|\Xc|})}}{(2L_P + 1)(1 + L_{\Pi}) -1}\Big((2 L_P + 1)^t(1 + L_{\Pi})^t  - 1\Big){ \sqrt{8\ln(2)/N}}.$ Therefore
\begin{eqnarray*}
& &  \sup_{\pi} \Big|u_N^\pi(\mu^N) - v^\pi(\mu^N)\Big| \leq (\tilde R + 2 L_{\tilde r})\sum_{t=0}^\infty \gamma^t\delta_t^{1, N} +  L_{\tilde r} \sum_{t=0}^\infty \gamma^t \delta_t^{2, N}\\
&\leq &  \Big\lbrace \frac{ (\tilde R + 2 L_{\tilde r}){ (L_P \sqrt{|\Uc|} + 2\sqrt{|\Xc|})}}{(2L_P + 1)(1 + L_{\Pi}) -1} \Big(\frac{1}{1- (2L_P +1) (1 +L_{\Pi})\gamma} - \frac{1}{1 -\gamma}\Big) + \frac{{ {\sqrt{|\Uc|}}} L_{\tilde r}}{1 -\gamma} \Big\rbrace { \sqrt{8\ln(2)/N}}.
\end{eqnarray*}

This proves \eqref{equvalueappro}.
\end{proof}

\subsection{Q-function Approximation under Learning
}\label{sec:learning_approx}
In this section we show that, with $\mathcal{O}(\log(1/\epsilon))$ samples and with $\epsilon$ the size of $\epsilon$-set, the kernel-based Q function from Algorithm \ref{QL_CD} provides an approximation to the Q function of  cooperative MARL, with an error of $\mathcal{O}(\epsilon+\frac{1}{\sqrt{N}})$,

For the \eqref{rewardagentj} problem specified in Section \ref{sec:value_approx} { and given the initial states $x^{j,N}$ and actions $u^{j,N}$ from all agents ($j=1,2,\dots,N$)}, let us
define the corresponding Q function,
\begin{eqnarray}\label{eq:Q_N}
Q^{\pi}_N(\mu^N,h^N) =  \frac{1}{N} \sum_{j=1}^N \tilde r(x^{j, N}, \mu^N, u^{j, N}, { \nu^N}) +\frac{1}{N} \sum_{j=1}^N\E\Big[\sum_{t=1}^\infty \gamma^t \tilde r(x_t^{j, N}, \mu_t^N, u_t^{j, N}, \nu_t^N)\Big]
\end{eqnarray}
subject to
\begin{eqnarray}
\;\; & x_{1}^{j, N} \sim P(x^{j, N}, \mu^N, u^{j, N}, { \nu^N}),\nonumber\\
 \;\; & x_{t + 1}^{j, N} \sim P(x_t^{j, N}, \mu_t^N, u_t^{j, N}, \nu_t^N), \; u_t^{j, N} \sim \pi(\mu_t^N, x_t^{j, N}), \;\;\; 1 \leq j \leq N, \,\,{\rm and }\,\, t \ge 1.\nonumber
\end{eqnarray}
where $\mu^N(x) = \frac{\sum_{j=1}^N 1({x^{j, N}}=x)}{N}$, { $\nu^N(u) = \frac{\sum_{j=1}^N 1(u^{j,N}=u)}{N}$} and $h^N(x)(u) = \frac{\sum_{j=1}^N 1(x^{j, N}=x;\,u^{j, N}=u)}{\sum_{j=1}^N 1(x^{j, N}=x)}$ with the convention $\frac{0}{0}=0$, and define
\begin{eqnarray}\label{eq:Q_N_optimal}
 Q_N(\mu^N,h^N) = \sup_{\pi}Q^{\pi}_N(\mu^N,h^N).
\end{eqnarray}

\begin{theorem}\label{thm:Q_approximation_Nplayer}Fix $\epsilon>0$. Assume the same conditions as in Theorem \ref{thm:conv_mfc}, Theorem \ref{NagentMFC} and Corollary \ref{corr:optimal_value_app}. Then 
there exists some $\widetilde{C} =\widetilde{C}(L_P,L_{\Pi},|\mathcal{X}|,{ |\mathcal{U}|},\tilde{R},L_{\tilde{r}},\gamma)>0 $, {depending on the dimensions of the state and action spaces in a sublinear order $(\sqrt{|\mathcal{X}|}+\sqrt{\mathcal|\Uc|})$}, such that
\begin{eqnarray}
||Q_N-\Gamma_K\widehat{Q}_\epsilon||_\infty\leq\frac{L_r + 2\gamma N_K L_K V_{\max} L_{\Phi}}{1-\gamma}\cdot\epsilon+\frac{2 L_r}{(1-\gamma L_{\Phi})(1-\gamma)}\cdot\epsilon +\frac{\widetilde{C}}{\sqrt{N}}.
\end{eqnarray}
\end{theorem}

Combining Theorem \ref{thm:sample_complexity} and Theorem \ref{thm:Q_approximation_Nplayer} implies the following:
fix any $\epsilon>0$, there exists an integer $D_{\epsilon}\in \mathbb{N}$ such that Algorithm \ref{QL_CD} outputs a kernel-based Q function with $C\log(1/\epsilon)$ samples. With high probability, this kernel-based Q function is  $\epsilon$ close  to the Q function of MARL when the agent number $N>D_{\epsilon}$. Here $C= C( L_P,L_{\Pi},|\mathcal{X}|,{ |\Uc|}, \tilde{R},L_{\tilde{r}},\gamma)$  is sublinear with respect to $|\mathcal{X}|$ and $|\mathcal{U}|$ and independent of the number of agents $N$. 

\begin{proof}[Proof of Theorem \ref{thm:Q_approximation_Nplayer}]
First we have
\begin{eqnarray}
 Q_N(\mu^N,h^N) =  \frac{1}{N} \sum_{j=1}^N \tilde r(x^{j, N}, \mu^N, u^{j, N}, { \nu^N}) +\gamma \sup_{\pi}\mathbb{E}[u_N^{\pi}(\mu^N_1)]
\end{eqnarray}
On the other hand,  by the definitions of $Q$ in \eqref{equmfcQ}, $\mu^N$ and $h^N$,
\beq
Q(\mu^N, h^N) &=&  \frac{1}{N} \sum_{j=1}^N \tilde r(x^{j, N}, \mu^N, u^{j, N},  { \nu^N})+\sup_{\pi\in\Pi}\bigg[\sum_{t=1}^{\infty} \gamma^{t} r(\mu_t, h_t)\bigg|\mu_1=\Phi(\mu^N, h^N)\bigg]\nonumber\\
&=&  \frac{1}{N} \sum_{j=1}^N \tilde r(x^{j, N}, \mu^N, u^{j, N},  { \nu^N})+\gamma\,\sup_{\pi\in\Pi}v^{\pi}(\Phi(\mu^N, h^N))
\enq
with $h_t = \pi_t(\mu_t)$. Therefore,
\begin{eqnarray}\label{eqn:QN_bound_main0}
&&|Q(\mu^N, h^N) - Q_N(\mu^N, h^N)| = \gamma \left|\sup_{\pi\in\Pi}v^{\pi}(\Phi(\mu^N, h^N)) -\sup_{\pi}\mathbb{E}[u_N^{\pi}(\mu^N_1)] \right|\nonumber\\
&\leq &\gamma \left|v^{\pi^*}(\Phi(\mu^N, h^N)) -\mathbb{E}[v^{\pi^*}(\mu^N_1)] \right| + \gamma\left| \mathbb{E}\left[v^{\pi^*}(\mu^N_1) -u_N^{\widetilde{\pi}}(\mu^N_1) \right]\right|\label{eq:Q_bound_1}
\end{eqnarray}
where $\pi^* \in \arg\sup_{\pi \in \Pi} v^{\pi}$, $\widetilde{\pi} \in \arg\sup_{\pi \in \Pi} u_N^{\pi}$, and the expectation in \eqref{eq:Q_bound_1} is taking with respect to $\mu^N_1$.

For the second term in \eqref{eq:Q_bound_1},
\begin{eqnarray}\label{eqn:QN_bound_main1}
\left| \mathbb{E}\left[v^{\pi^*}(\mu^N_1) -u_N^{\widetilde{\pi}}(\mu^N_1) \right]\right| \leq \mathbb{E} \left| \left[v^{\pi^*}(\mu^N_1) -u_N^{\widetilde{\pi}}(\mu^N_1) \right]\right| \leq \frac{C}{\sqrt{N}},
\end{eqnarray}
in which the first inequality holds by convexity and the second inequality holds due to Corollary \ref{corr:optimal_value_app}.

For the first term in \eqref{eq:Q_bound_1},
\begin{eqnarray}
&&\left|v^{\pi^*}(\Phi(\mu^N, h^N)) -\mathbb{E}_{\mu_1^N}[v^{\pi^*}(\mu^N_1)] \right|\nonumber\\
&\leq& (\tilde{R}+2L_{\tilde{r}}) \,\sum_{t=0}^{\infty} \gamma^t \E \left[\left\|\mu_t^{\pi^*}-\overline{\mu}_t^{\pi^*}\right\|_1 \right]  +L_{\tilde{r}}\, \sum_{t=0}^\infty \gamma^t \E\Big[\left\|\nu_t^{\pi^*} - \overline{\nu}_t^{ \pi^*}\right\|_1\Big]\label{eq:v_bound_1}\\
&\leq &{ (\tilde{R}+3L_{\tilde{r}}+L_{\tilde{r}}L_\Pi)} \,\sum_{t=0}^{\infty} \gamma^t \E \left[\left\|\mu_t^{\pi^*}-\overline{\mu}_t^{\pi^*}\right\|_1 \right], \label{eq:v_bound_2}
\end{eqnarray}
in which  $\mu_{t+1}^{\pi^*} = \Phi(\mu_{t}^{\pi^*},\pi^*(\mu_{t}^{\pi^*}))$ with initial condition $\mu_0^{\pi^*}=\Phi(\mu^N,h^N)$, $\overline{\mu}_{t+1}^{\pi^*} = \Phi(\overline{\mu}_{t}^{\pi^*},\pi^*(\overline{\mu}_{t}^{\pi^*}))$ with initial condition $\overline{\mu}_0^{\pi^*}=\mu_1^N$. In addition, $\nu_t^{\pi^*}(u) = \sum_{x\in \mathcal{X}}\pi^*(\mu_t^{\pi^*},x)(u)\mu_t^{\pi^*}(x)$ and $\overline\nu_t^{\pi^*}(u) = \sum_{x\in \mathcal{X}}\pi^*(\overline{\mu}_t^{\pi^*}, x)(u)\overline{\mu}_t^{\pi^*}(x)$. \eqref{eq:v_bound_1} holds by the continuity of $r$ from Lemma \ref{lemma:cont_r} and Assumption \ref{ass:r_MFC}. { \eqref{eq:v_bound_2} holds since by Lemma \ref{lemma:cont_nu} and Assumption \ref{ass: pi},
\begin{eqnarray*}
\left\|\nu_t^{\pi^*}-\overline{\nu}_t^{\pi^*}\right\|_1 &\leq& \left\|\mu_t^{\pi^*}-\overline{\mu}_t^{\pi^*}\right\|_1 + \max_{x\in\Xc}\left\|\pi^*(\mu_t^{\pi^*}, x)-\pi^*(\overline{\mu}_t^{\pi^*}, x)\right\|_1\leq
\left(1+L_{\Pi}\right)\left\|\mu_t^{\pi^*}-\overline{\mu}_t^{\pi^*}\right\|_1.
\end{eqnarray*}
For $t=0$, 
\begin{eqnarray}\label{eq:initial_ineq}
& & \E\left[\left\|\mu_0^{\pi^*}-\overline\mu_0^{\pi^*}\right\|_1\right] = \E\left[\left\|\mu_1^N-\Phi(\mu^N,h^N)\right\|_1\right] \\
&= &\E \left[\left\|\mu_1^N-\frac{1}{N}\sum_{j=1}^N \sum_{u \in \Uc}P(x^{j,N}, \mu^N, u, \nu^N)(x)h^N(x^{j, N})(u) \right\|_1\right]\nonumber \\
&=& \E\left[\sup_{g: \Xc \to \{-1, 1\}} \frac{1}{N} \sum_{j=1}^N g(x_1^{j, N}) - \frac{1}{N}\sum_{j=1}^N \sum_{x \in \Xc}\sum_{u \in \Uc}P(x^{j,N}, \mu^N, u, \nu^N)(x)h^N(x^{j, N})(u) g(x) \right] \nonumber \\
&\leq& \sqrt{8 |\Xc|\ln(2)/N},\nonumber 
\end{eqnarray}
where the second equality is by $\nu^N(u) = \sum_{x \in \Xc}\mu^N(x) h^N(x)(u)$ and by the definition of $\Phi$, and in the last inequality, $\{g(x_1^{j, N}) - \sum_{x \in \Xc}\sum_{u \in \Uc}P(x^{j,N}, \mu^N, u, \nu^N)(x)h^N(x^{j, N})(u) g(x)\}_{j=1}^N$ are independent mean-zero sub-Gaussian random variables bounded by $[-2, 2]$ and thus we proceed the similar arguments as \reff{deltanu}. 
}

We now prove by induction it holds for all $t \ge 0$ that 
\begin{eqnarray}
\label{eq:induction}
    \E\left[\|\mu_t^{\pi^*}-\overline\mu_t^{\pi^*}\|_1\right] \leq ((2L_P+1)( L_{\Pi}+1))^t { \sqrt{8|\Xc|\ln(2)/N}}.
\end{eqnarray}
\eqref{eq:induction} holds when $t=0$ given \eqref{eq:initial_ineq}. Now assume \eqref{eq:induction} holds for $t \leq s$. When $t=s+1$, we have
\begin{eqnarray}
\E\left[\|\mu_{s+1}^{\pi^*}-\overline\mu_{s+1}^{\pi^*}\|_1\right] &=& \E\left[\left\|\Phi(\mu_{s}^{\pi^*},\pi^*(\mu_{s}^{\pi^*}))-\Phi(\overline{\mu}_{s}^{\pi^*},\pi^*(\overline{\mu}_{s}^{\pi^*}))\right\|_1\right]\nonumber\\
&\leq& (2L_P +1) d_{\mathcal{C}} \Bigg(\Big(\mu_{s}^{\pi^*},\pi^*(\mu_{s}^{\pi^*})\Big)\,\,,\Big(\overline{\mu}_{s}^{\pi^*},\pi^*(\overline{\mu}_{s}^{\pi^*})\Big)\Bigg)\nonumber \\
&=& (2L_P +1) \Big(\|\mu_{s}^{\pi^*} - \overline\mu_{s}^{\pi^*})\|_1+\|\pi^*(\mu_{s}^{\pi^*}) - \pi^*(\overline\mu_{s}^{\pi^*})\|_1 \Big)\nonumber\\
&\leq & (2L_P +1)(1+L_{\Pi}) \|\mu_{s}^{\pi^*} - \overline\mu_{s}^{\pi^*}\|_1 \nonumber\\
&\leq& ((2L_P +1)(1+L_{\Pi}))^{s+1} { \sqrt{8|\Xc|\ln(2)/N}},
\end{eqnarray}
where the first inequality holds by Lemma \ref{lemma:cont_phi} and the second inequality holds by Assumption \ref{ass: pi}, and the third inequality holds by induction. Finally when $(2L_P +1)(1+L_{\Pi})\gamma<1$,
\begin{eqnarray}\label{eqn:QN_bound_main2}
\eqref{eq:v_bound_2} &\leq& { (\tilde{R}+3L_{\tilde{r}}+L_{\tilde{r}}L_\Pi)}\sum_{t=0}^{\infty} \sqrt{8 |\Xc|\ln(2)|/N}((2L_P +1)(1+L_{\Pi})\gamma)^t\\
&=& \sqrt{8|\Xc|\ln(2)|/N}{ \frac{\tilde{R}+3L_{\tilde{r}}+L_{\tilde{r}}L_\Pi}{1-(2L_P+1)(1+L_{\Pi})\gamma}}\nonumber.
\end{eqnarray}
Therefore, combining \eqref{eqn:QN_bound_main0}, \eqref{eqn:QN_bound_main1} and \eqref{eqn:QN_bound_main2}, we have proven that there exists some $\widetilde{C} =\widetilde{C}(L_P,L_{\Pi},|\mathcal{X}|,|\Uc|, \tilde{R},L_{\tilde{r}},\gamma)>0 $ such that $\left\|Q - Q_N\right\|_\infty\leq\frac{\widetilde{C}}{\sqrt{N}}$. Here $\widetilde{C}$ {depends on the dimensions of the state and action spaces in a sublinear order $(\sqrt{|\mathcal{X}|}+\sqrt{\mathcal|\Uc|})$} and is independent of the number of agents $N$. 
Theorem \ref{thm:Q_approximation_Nplayer} follows from combining the result above with Theorem \ref{thm:conv_mfc}.
\end{proof}

\section{Experiments}\label{sec:experiments}

We will test the MFC-K-Q algorithm on a network traffic congestion control problem.
In the network there are senders and receivers. Multiple senders share a single communication link which has an unknown and limited bandwidth. When the total sending rates from these senders exceed the shared bandwidth, packages may be lost. Sender streams data packets to the receiver and receives feedback from the receiver on success or failure in the form of packet acknowledgements (ACKs).
(See Figure \ref{fig:network_illustration} for illustration and \cite{JRGST2019} for a similar set-up).  The control problem for each sender is to send the packets  as fast as possible and with the risk of packet loss as little as possible. Given a large interactive population of senders, the exact dynamics of the system and the rewards are unknown, thus it is natural to formulate this control problem in the framework of learning MFC. 

\begin{figure}[H] 
  \centering
  \includegraphics[width=0.5\linewidth]{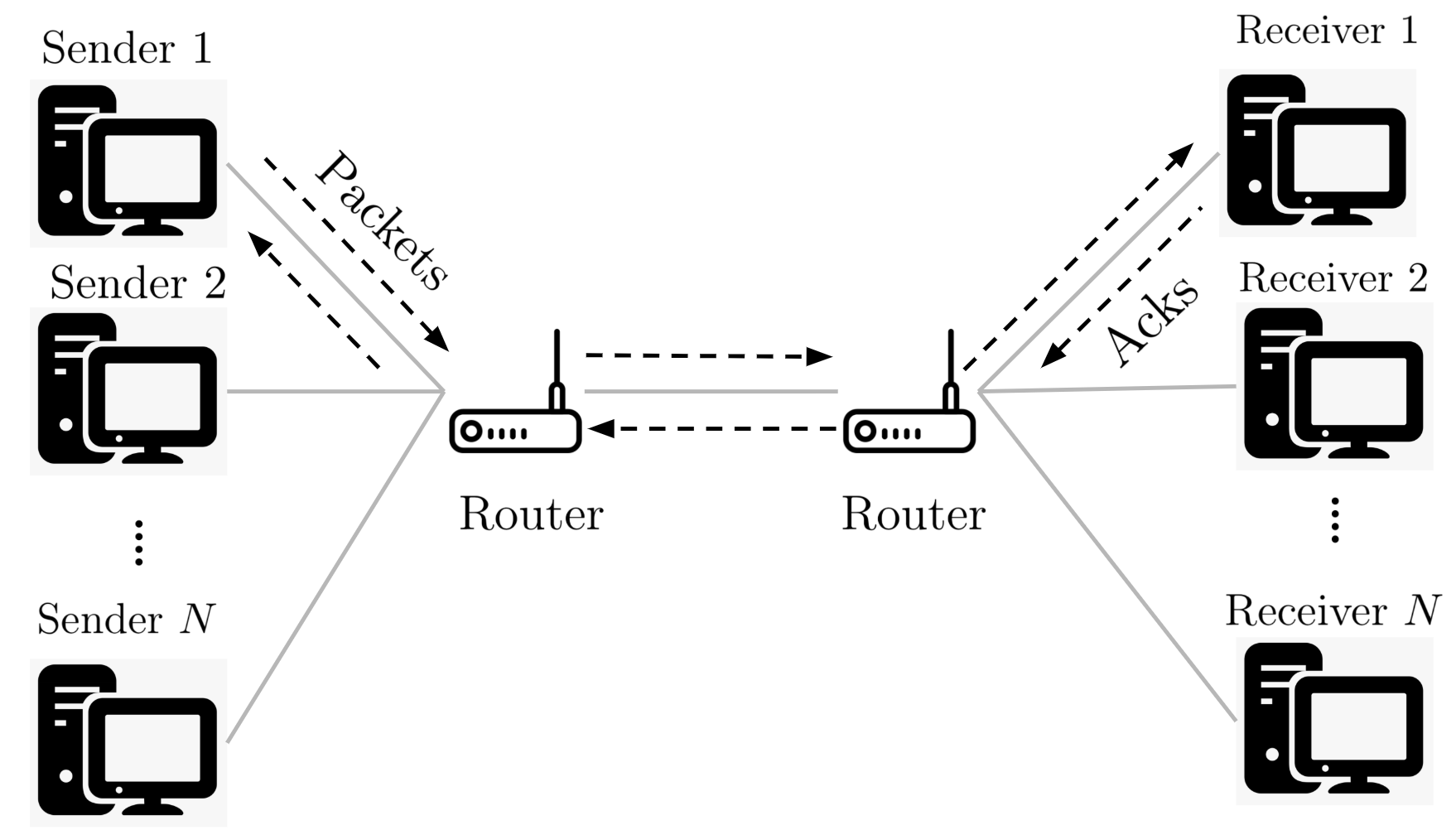}
  \caption{ \label{fig:network_illustration} Multiple network traffic flows sharing the same link.}
   \vskip -0.2in
\end{figure}

\subsection{Set-up}\label{experiment-set-up}
\paragraph{States} For a representative agent in MFC problem with learning, at the beginning of each round $t$, the state $x_t$ is her inventory (current unsent packet units) taking values from $\mathcal{X}=\{0, \ldots, |\Xc|-1\}$. Denote $\mu_t: = \{\mu_t(x)\}_{x \in \mathcal{X}}$ as the population state distribution over $\mathcal{X}$.

\paragraph{Actions}  The action is the sending rate. At the beginning of each round $t$, the agent can adjust her sending rate $u_t$, which remains fixed in $[t,t+1)$. Here we  assume   $u_t\in \mathcal{U}=\{0,\ldots,|\Uc|-1\}$.  Denote $h_t=\{h_t(x)(u)\}_{x\in \mathcal{X},u \in \mathcal{U}}$ as the policy from the central controller.  

\paragraph{Limited bandwidth and packet loss}
A system with $N$ agents has a shared link of unknown bandwidth $cN$ {($c>0$).  In the mean-field limit with $N \rightarrow \infty$, $F_t =\sum_{x \in \mathcal{X}, u \in \mathcal{U}} u {h_t(x)(u)} \mu_t(x)$ is  the average sending rate at time $t$.  If $F_t>c$, with probability $\frac{(F_t-c)}{F_t}$, each agent's packet will be lost.

\paragraph{MFC dynamics}  At time $t+1$, the state of the representative agent moves from $x_t$ to $x_t-u_t$. {Overshooting is not allowed: $u_t\leq x_t$.} Meanwhile, at the end of each round,
there are some packets added to each agent's packet sending queue. The packet fulfillment consists of two scenarios. First a lost package will be added to the original queue. Then once the inventory hits zero,  a random fulfillment with uniform distribution {Unif$(\mathcal{X})$} will be added to her queue. That is, $x_{t+1} = x_t-u_t + u_t 1_{t}(L) + (1-1_{t}(L)1(u_t=x_t)\cdot U_t,$
where $1_{t}(L) = {1}(\mbox{packet is lost in round t})$, with 1 an indicator function and {$U_t$ $\sim$ Unif($\mathcal{X})$}.

\paragraph{Evolution of population state distribution $\mu_t$} 
Define, for $x \in \mathcal{X}$,
\begin{eqnarray*}
\tilde{\mu}_t(x) = &\sum_{x^{\prime} \geq x}\mu_t(x^{\prime}) h_t(x^{\prime})(x^{\prime}-x)\left(1- 1(F_t>c)\frac{F_t-c}{F_t}\right) + \mu_t(x)1(F_t>c)\frac{F_t-c}{F_t}.
\end{eqnarray*}
Then $\tilde{\mu}_t$ represents the state of the population distribution after the first step of task fulfillment  and before the second step {{of}} task fulfillment.
Finally, for $x \in \mathcal{X}$,
${\mu}_{t+1}(x) = \left(\tilde{\mu}_t(x) +\frac{\tilde{\mu}_t(0)}{|\Xc|}\right) {1}(x \neq 0) + \frac{\tilde{\mu}_t(0)}{|\Xc|} {1}(x = 0),$
describes the transition of the flows $\mu_{t+1}=\Phi(\mu_t, h_t)$. 

\paragraph{Rewards} Consistent with \cite{DMZAGGS2018} and \cite{JRGST2019}, the reward function depending on throughput, latency, with loss penalty is defined as
$\tilde r = a*\mbox{throughput}-b*\mbox{latency}^2 -d*\mbox{loss},$
with $a,b,d \ge 0$.

\subsection{Performance of MFC-K-Q Algorithm}
We first test the convergence property and performance of MFC-K-Q (Algorithm \ref{QL_CD}) for this traffic control problem with different kernel choices and with varying $N$. We then compare MFC-K-Q with MFQ Algorithm \cite{CLT2019b} on MFC, Deep PPQ  \cite{JRGST2019}, and PCC-VIVACE  \cite{DMZAGGS2018} on MARL.

We assume the access to {an MFC} simulator $\mathcal{G}(\mu,h)=(\mu^{\prime},{r})$. That is, for any pair $(\mu,h)$ $\in$ $\Cc$, we can sample  the aggregated population reward $ r$ and the next population state distribution $\mu^{\prime}$  under policy $h$.   We sample $\mathcal{G}(\mu,h)=(\mu^{\prime},{r})$  once for all $(\mu,h)\in \mathcal{C}_{\epsilon}$. 
In each outer iteration, each update on $(\mu,h)\in \mathcal{C}_{\epsilon}$ is one inner-iteration. 
Therefore, the total number of inner iterations within each outer iteration equals $|\mathcal{C}_{\epsilon}|$. 

\paragraph{Applying MFC policy to $N$-agent game} To measure the performance of the MFC policy  $\pi$ for an $N$-agent set-up, we apply $\pi$ to the empirical state distribution of $N$ agents. 
 
\paragraph{Performance criteria} We assume the access to an N-agent simulator $\mathcal{G}^N(\pmb{x},\pmb{u})=(\pmb{x}^{\prime},\pmb{r})$.
That is, if agents take joint action $\pmb{u}$ from state $\pmb{x}$, we can observe the joint reward $\pmb{r}$ and the next joint state $\pmb{x}^{\prime}$. We evaluate different policies in the $N$-agent environment. 

We randomly sample $K$ initial states {{$\{\pmb{x}_0^k\in \mathcal{X}^N\}_{k=1}^{K}$}} and apply policy ${\pi}$ to each initial state $\pmb{x}_0^k$ and collect the continuum rewards in each path for $T_0$ rounds $\{\bar{r}_{k,t}^{\pi}\}_{t=1}^{T_0}$. {Here $\bar{r}_{k,t}^{\pi} = \frac{\sum_{i=1}^N r^{\pi,i}_k}{N}$ is the average reward from $N$ agents in round $t$  under policy $\pi$.}
Then $R_N^{\pi}(\pmb{x}_0^k):=\sum_{t=1}^{T_0}\gamma^{t}\bar{r}_{k,t}^{\pi}$ is used to approximate the value function $V_{\mathcal{C}}^{\pi}$ with policy $\pi$, when $T_0$ is large. 

Two performance criteria are used: the   first one $C^{(1)}_N({\pi}) = \frac{1}{K}\sum_{k=1}^K R_N^{\pi}(\pmb{x}^k_0)$
measures the average reward from policy ${\pi}$; and the second criterion
$C^{(2)}_N({\pi}^1,{\pi}^2) = \frac{1}{K}\sum_{k=1}^K\frac{ R_N^{\pi_1}(\pmb{x}^k_0)-R_N^{\pi_2}(\pmb{x}^k_0)}{R_N^{\pi_1}(\pmb{x}^k_0)}$
measures the relative improvements of using policy ${\pi}^1$ instead of policy ${\pi}^2$.

\paragraph{Experiment set-up} We set $\gamma=0.5$, $a=30$, $b=10$, $d=50$, $c=0.4$, $M=2$, $K=500$ and $T_0=30$, and compare policies with $N=5n$ agents $(n=1,2,\cdots,20)$. For the $\epsilon$-net, we take uniform grids with $\epsilon$ distance between adjacent points on the net. {The confidence intervals are calculated with $20$ repeated experiments.}

\begin{figure}[tbhp]
\centering
\subfloat[Convergence of Q function.]{\label{fig:kernel_convergence}\includegraphics[width=0.45\linewidth]{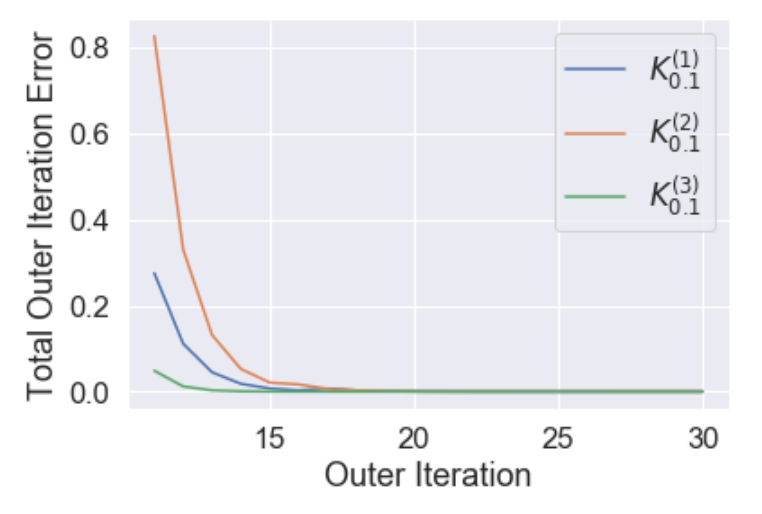}}
\subfloat[$C_N^{(1)}$: Average reward.]{\label{fig:kernel_reward}\includegraphics[width=0.45\linewidth]{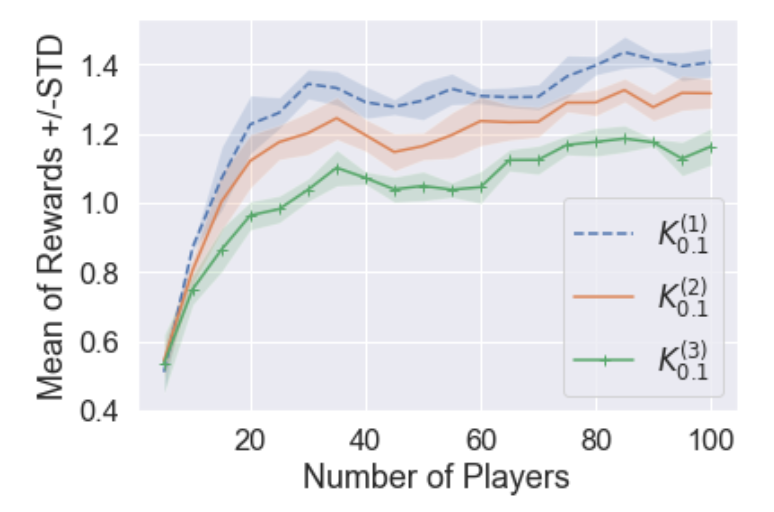}}
\caption{Performance comparison among different kernels.}
\label{fig:kernel}
\end{figure}

\paragraph{Results with different kernels}
We use the following kernels with hyper-parameter $\epsilon$: triangular, (truncated) Gaussian, and (truncated) constant kernels. That is, 
 $\phi^{(1)}_\epsilon(x,y) = \textbf{1}_{\{\|x-y\|_2\leq \epsilon\}}\big|\epsilon-\|x-y\|_2\big|$,
     $\phi^{(2)}_\epsilon(x,y)=\textbf{1}_{\{\|x-y\|_2\leq \epsilon\}}\frac{1}{\sqrt{2\pi}}\exp(-|\epsilon-\|x-y\|_2|^2)$,
    and  $\phi^{(3)}_\epsilon(x,y) = \textbf{1}_{\{\|x-y\|_2\leq \epsilon\}}$.
    We run the experiments for 
$ K_\epsilon^{(j)}(c^i, c) = \frac{\phi_\epsilon^{(j)}(c^i, c)}{\sum_{i=1}^{N_\epsilon} \phi_\epsilon^{(j)}(c^i, c)}. $
with $j=1, 2, 3$ and $\epsilon=0.1$.

All kernels lead to the convergence of Q functions within $15$ outer iterations (Figure \ref{fig:kernel_convergence}). When $N \leq 10$, the performances of all kernels are similar since $\epsilon$-net is accurate for games with $N=\frac{1}{\epsilon}$ agents. When $N \geq 15$, $K_{0.1}^{(1)}$ performs the best  and  $K_{0.1}^{(3)}$ does the worst (Figure \ref{fig:kernel_reward}): treating all nearby $\epsilon$-net points with equal weights yields relatively poor performance. 
 
Further comparison of $K_{0.1}^{(j)}$'s suggests that appropriate choices of kernels for specific problems with particular structures of Q functions help reducing errors  from a fixed $\epsilon$-net. 
 
\begin{figure}[tbhp]
\centering
\subfloat[Convergence of Q function]{\includegraphics[width=0.45\linewidth]{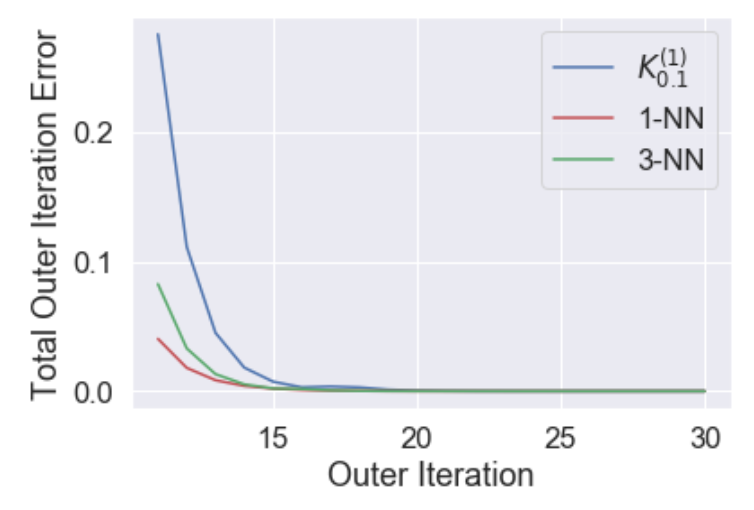}}
\subfloat[$C_N^{(1)}$: Average reward.]{\label{fig:epsion_performance_2}\includegraphics[width=0.45\linewidth]{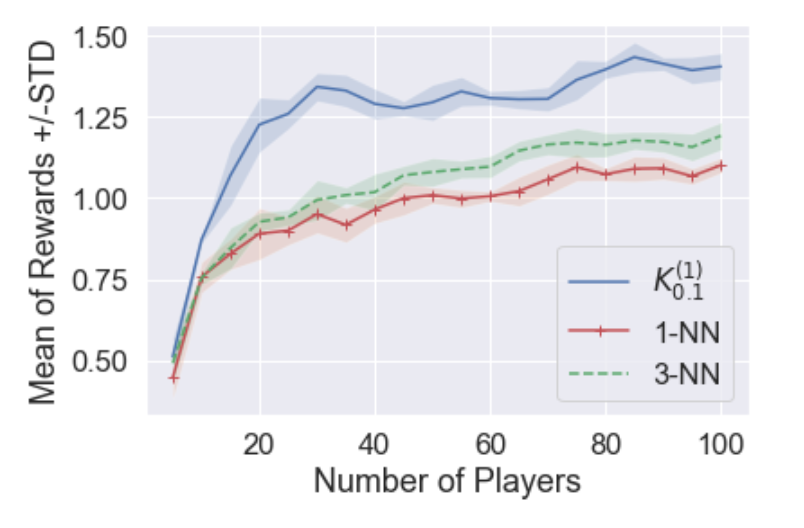}}
\hfill
\centering
\subfloat[$C_N^{(2)}$: Improvement of $K_{0.1}^{(1)}$ from $1$-NN.]{\label{fig:epsion_performance_3}\includegraphics[width=0.45\linewidth]{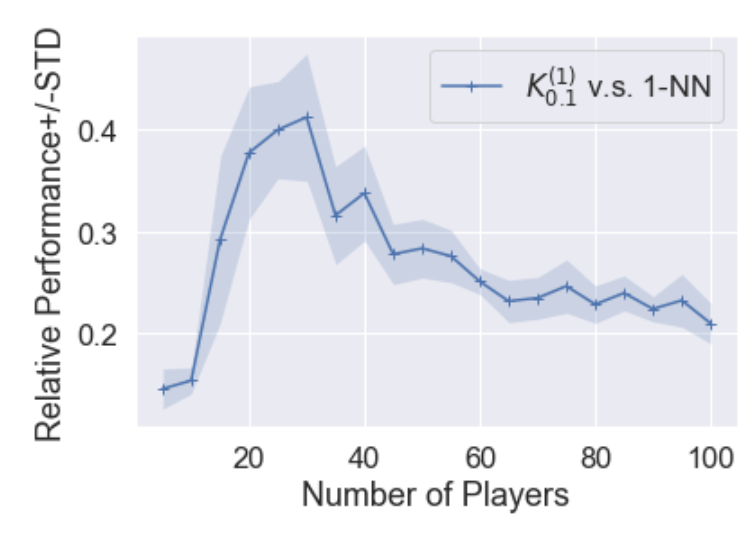}}
\subfloat[$C_N^{(2)}$: Improvement of $K_{0.1}^{(1)}$ from $3$-NN.]{\label{fig:truncation_error_2}\includegraphics[width=0.45\linewidth]{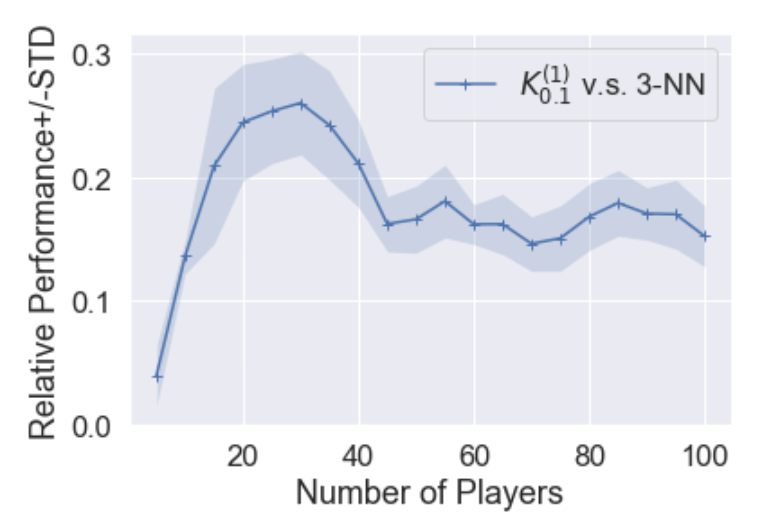}}
\caption{Comparison between $K_{0.1}^{1}(x,y)$ and $k$-NN ($k=1,3$).}
\end{figure}
\paragraph{Results with different $k$-nearest neighbors}
We compare kernel $K_{0.1}^{(1)}(x,y)$ with the $k$-nearest-neighbor ($k$-NN) method ($k=1,3$), with $1$-NN the projection approach by which each point is projected onto the closest point in $\mathcal{C}_{\epsilon}$, a simple method for  continuous state and action spaces \cite{MM2002,V2012}.

All $K_{0.1}^{(1)}(x,y)$ and $k$-NN  converge within 15 outer iterations. The performances of $K_{0.1}^{1}(x,y)$ and $k$-NN are similar when $N\leq 10$. However,  $K_{0.1}^{(1)}(x,y)$ outperforms both $1$-NN and $3$-NN for large $N$ under both criteria $C_N^{(1)}$ and $C_N^{(2)}$:  under $C_N^{(1)}$, $K_{0.1}^{(1)}(x,y)$, $1$-NN, and $3$-NN have respectively average rewards of $1.4$,  $1.07$, and $1.2$ when $N\ge 65$; 
under $C_N^{(2)}$, $K_{0.1}^{(1)}(x,y)$ outperforms $1$-NN and $3$-NN by 15$\%$ and 13$\%$ respectively when $N=10$, by  29$\%$ and 21$\%$ respectively when $N=15$, and by  25$\%$ and 16$\%$ respectively when $N \ge 60$.


\paragraph{Comparison with other algorithms}
We compare MFC-K-Q with $K_{0.1}^{(1)}$ with three representative algorithms, MFQ  from \cite{CLT2019b}, Deep PPQ  from \cite{JRGST2019}, and PCC-VIVACE  from \cite{DMZAGGS2018} on MARL.
Our experiment demonstrates superior performances of MFC-K-Q.
\begin{itemize}
    \item When N>40, MFC-K-Q dominates all these three algorithms (Figure~\ref{fig:reward_comparison}) and it learns the bandwidth parameter $c$ most accurately (Figure~\ref{fig:bandwidth_comparison}). Despite being the best performer when N<35, Deep PPQ  suffers from the ``curse of dimensionality'' and the performance gets increasingly worse when N increases;
    \item MFC-K-Q with $K_{0.1}^{(1)}$ dominates MFQ, which is similar to our worst performer MFC-K-Q with 1-NN. In general, kernel regression performs better than simple projection (adopted in MFQ) where only one point is used to estimate $Q$;
    \item the decentralized PCC-VIVACE has the worst performance. Moreover, it is insensitive to the bandwidth parameter $c$. See Figure~\ref{fig:bandwidth_comparison}.
\end{itemize}
\begin{figure}[tbhp]
\centering
\subfloat[$C_N^{(1)}$: Average reward.]{\label{fig:reward_comparison}\includegraphics[width=0.45\linewidth]{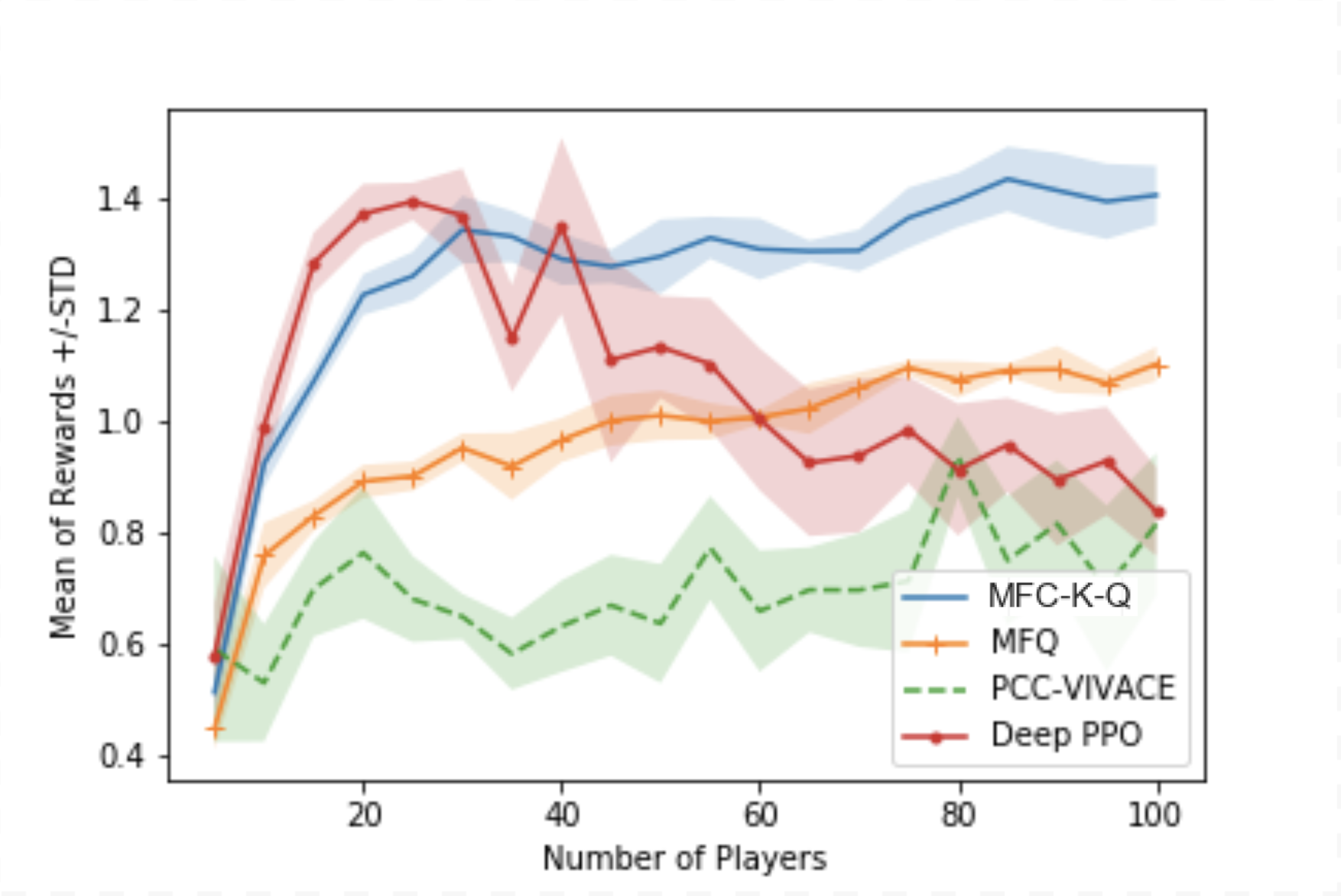}}
\subfloat[Average sending flow.]{\label{fig:bandwidth_comparison}\includegraphics[width=0.45\linewidth]{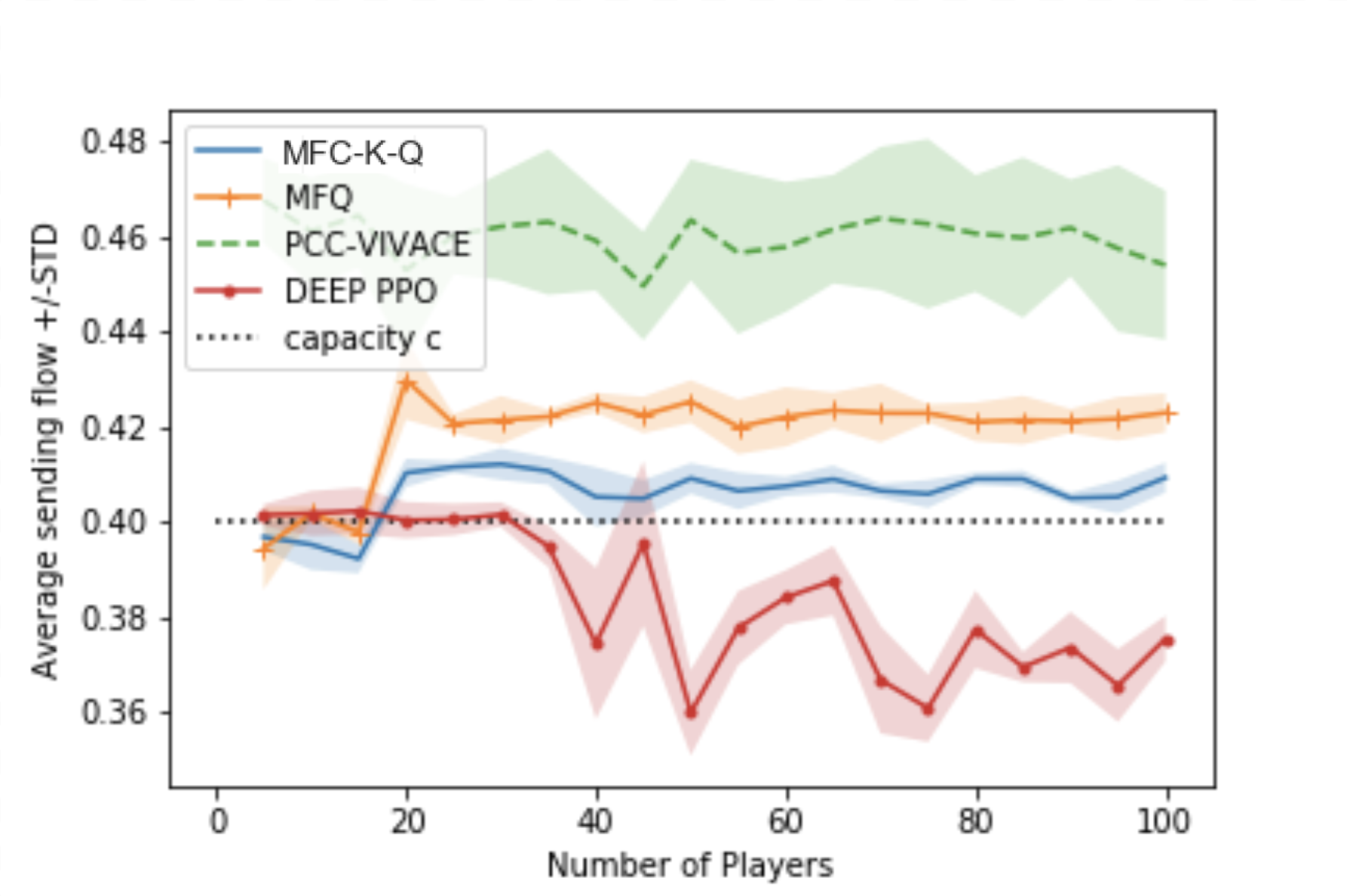}}
\caption{\footnotesize Performance comparison among different algorithms.}
\label{fig:algos}
\end{figure}

\section{Discussions and Future Works}\label{sec:discussion}
\paragraph{Related works on kernel-based reinforcement learning}

Kernel method is a popular dimension reduction technique to map high-dimensional features into a  low dimension space that best represents the original features. This technique was first introduced for RL by \cite{OS2002, OG2002},  in which a kernel-based reinforcement learning algorithm (KBRL) was proposed to handle the continuity of the state space. 
Subsequent works demonstrated the applicability of KBRL to large-scale problems and for various types of RL algorithms (\cite{BPP2016}, \cite{taylor2009kernelized} and \cite{xu2007kernel}). 
However, there is no prior work on  convergence rate or sample complexity analysis. 

Our kernel regression idea is closely related to \cite{SX2018},  which combined Q-learning with kernel-based nearest neighbor regression to study continuous-state stochastic MDPs with sample complexity guarantee. 
However, our problem setting and technique for error bound analysis are different from theirs. 
In particular, Theorem \ref{thm:conv_mfc} has both action space approximation and state space approximation; whereas \cite{SX2018} has only state space approximation and their action space is finite. The error control in \cite{SX2018} was obtained via martingale concentration inequalities whereas ours is by the regularity property of the underlying dynamics. Other than the kernel regression method, one could also consider the empirical (or approximate) dynamic programming approach to handle the infinite dimensional problem \cite{chen2009approximate,haskell2016empirical}.
\paragraph{Stochastic vs deterministic dynamics}
We reiterate that unlike learning algorithms for  stochastic dynamics where the choice of  learning rate $\eta_t$ is to guarantee the convergence of the Q function (see e.g. \cite{watkins1992q}),  MFC-K-Q  directly conducts the fixed point iteration for the approximated Bellman operator $B_\epsilon$ on the sampled data set, and sets the learning rate as  $1$ to fully utilize the deterministic nature of the dynamics. Consequently,  complexity analysis of this algorithm is reduced significantly. By comparison, for stochastic systems each component in the $\epsilon$-net has to be visited  sufficiently many times for a decent estimate in Q-learning.

\paragraph{Sample complexity comparison} Theorem~\ref{thm:sample_complexity} shows that  sample complexity for MFC with learning is   $\Omega(\text{poly}((1/\epsilon)\cdot\log(1/\delta)))$, instead of the exponential rate in $N$ by existing algorithms for cooperative MARL in Proposition~\ref{lemma:N_agent_complexity}.
Careful readings  reveal that this complexity analysis holds for other exploration schemes,
including the Gaussian exploration and the Boltzmann exploration, as long as Lemma \ref{lemma:covering} holds.

\paragraph{Convergence under different norms} Our main assumptions and results adopt the infinity norm ($\|\cdot\|_{\infty}$)  for ease of exposition. Under appropriate assumptions on the mixing behavior of the mean-field dynamic, and applying techniques in \cite{munos2008}, the convergence results can also be established under the $L_p$ ($\|\cdot\|_p$) norm to allow for the function approximation of Q-learning. In addition, by properly controlling the Lipschtiz constant, the empirical performance of the neural network approximation may be further improved (\cite{asadi2018lipschitz}).

\paragraph{Extensions to other settings} For future research, we are interested in extending our framework and learning algorithm to other variations of mean-field controls including risk-sensitive mean-field controls (\cite{bensoussan2017risk}, \cite{DT2016}, and \cite{DTT2015}),
robust mean-field controls (\cite{wang2020robust}), mean-field controls on polish space (\cite{saldi2020discrete}), and partially observed mean-field controls (\cite{DT2016,saldi2019approximate}).

 If the state space of each individual player is a Polish space \cite{saldi2020discrete}, one can adopt, instead of the Q learning framework in this paper,  Proximal Policy Optimization (PPO)  type of algorithms \cite{schulman2017proximal,schulman2015trust}. In this framework, the mean-field information on the lifted probability measure may be incorporated via a mean embedding technique, which embeds the mean-field states into a reproducing kernel Hilbert space (RKHS) \cite{smola2007hilbert,gretton2008kernel}.
 
 Given the connection between the Q function and the Hamiltonian of  nonlinear control problem with single-agent \cite{mehta2009q}, one may also extend the kernel-based Q learning algorithm to more general nonlinear mean-field control problems. \\

{ \noindent {\bf Acknowledgement.} We are grateful to two anonymous referees from SIAM Journal on Mathematics of Data Science for their detailed suggestions, which help improve the exposition of the paper and in particular  Section \ref{sec:connection}. We thank the authors of  \cite{Mondal2021} who spotted an error in the earlier proof of Theorem \ref{NagentMFC} in our original arxiv version; our correction of the error yields a refined upper bound of \eqref{eq:bound} {and leads a sublinear dependence  on the dimensions of the state and action spaces with order $(\sqrt{|\mathcal{X}|}+\sqrt{\mathcal|\Uc|})$ in the constant term $C$ of Theorem \ref{NagentMFC} }.
}

\bibliographystyle{siamplain}
\bibliography{refs}
\newpage
\appendix
\section{Table of Parameters}\label{app:table}
\begin{table}[H]
  \begin{center}
    \begin{tabular}{|r|l|} 
      \hline
      \textbf{Notation} & \textbf{Definition} \\
      \hline
       $\mathbb{R}^{\mathcal{S}}$ & set of all real-valued measurable functions on measurable space $\mathcal{S}$\\
       $\mathcal{P}(\mathcal{S})$ & set of all probability measures on $\mathcal{S}$\\
       $d_{\mathcal{P}(\mathcal{S})}$ & metric induced by $l_1$ norm: $d_{\mathcal{P}(\mathcal{S})}(u,v) = \sum_{s\in \mathcal{S}}|u(s)-v(s)|$ for any $u,v \in \mathcal{P}(\mathcal{S})$\\
        $\gamma$ & discount factor \\
       $1(x\in A)$ & indicator function of event $\{x\in A\}$\\
        $N$ & number of agents \\
        $\mathcal{X}$ & state space of single agent\\
          $\mathcal{S}$ & action space of single agent\\
    $\mu_t^N\in\mathcal{P}(\mathcal{X})$ & empirical state distribution of $N$ agents at time $t$\\
    $\nu_t^N\in\mathcal{P}(\mathcal{U})$ & empirical action distribution of $N$ agents at time $t$\\
     $\mu_t\in\mathcal{P}(\mathcal{X})$ &  state distribution of the MFC problem at time $t$\\
    $\nu_t\in \mathcal{P}(\mathcal{U})$ & action distribution of of the MFC problem at time $t$\\
    $\mathcal{H}$ & $\mathcal{H}:=\{h:\mathcal{X}\rightarrow
    \mathcal{P}(\mathcal{U})\}$ is the set of local policies\\
    $\mathcal{C}$ & $\mathcal{C}:=\Pc(\Xc)\times\Hc$, the product space of $\Pc(\Xc)$ and $\Hc$\\
    $\Pi$ & $\Pi:=\{\pi=\{\pi_t\}_{t=0}^\infty \,|\, \pi_t:\Pc(\Xc)\to\Hc\}$ is the set of admissible policies\\
    $\tilde{r}(x, \mu, u, {\nu}(\mu, h))$& individual reward\\
     $R$ &  bound of the reward, i.e., $|\tilde{r}|<R$\\
  $r(\mu, h)$& aggregated population reward $r(\mu, h):=\sum_{x \in \mathcal{X}} \sum_{u \in \mathcal{U}} \tilde r(x, \mu, u, {\nu}(\mu, h)) \mu(x) h(x)(u)$\\
  $L_P$ & Lipschitz constant for transition matrix $P$\\
  $L_{\tilde{r}}$ & Lipschitz constant for reward $\tilde{r}$\\
  $L_r:=\tilde{R}+2L_{\tilde{r}}$ &Lipschitz constant for $r$ \\  $L_\Phi:=2L_P+1$ &Lipschitz constant for $\Phi$\\
  $\mathcal{C}_{\epsilon} $ & $\epsilon$-net on $\mathcal{C}$\\
  $N_\epsilon$ & size of the $\epsilon$-net $\mathcal{C}_{\epsilon}$ on $\mathcal{C}$\\
  $N_{\Hc_\epsilon}$ & size of the $\epsilon$-net $\mathcal{H}_{\epsilon}$ on $\mathcal{H}$\\
  $K(c^i,c)$& weighted kernel function with $c^i \in \mathcal{C}_{\epsilon} $ and $c\in \mathcal{C}$\\
  $L_K$&Lipschitz constant for kernel $K$\\
  $N_K$& at most $N_K$ number of $c^i\in\Cc_\epsilon$ satisfies $K(c,c^i)>0$\\
  $\Gamma_K$& kernel regression operator from $\R^{\Cc_\epsilon}\to\R^{\Cc}$\\
  $T_{\Cc,\pi}$& covering time of the $\epsilon$-net under policy $\pi$\\
  $M_\epsilon$& constant appearing in Assumption \ref{ass:dynamic}, the controllability of the dynamics\\

         \hline
    \end{tabular}
   
     \caption{Summary of parameters for LMFC and MARL.} \label{tab:parameters}
  \end{center}
\end{table}

\section{Proofs of Lemmas}\label{app:proof}

\begin{proof}[Proof of Lemma \ref{lemma:flowmut}]
At time step $t$, assume $x_t\sim\mu_t$. Under the policy $\pi_t$, it is easy to check via direct computation that the corresponding action distribution $\nu_t$ is $\nu(\mu_t, \pi_t(\cdot, \mu_t))$. Meanwhile, for any bounded function $\varphi$ on $\Xc$, by the law of iterated conditional expectation:
\begin{eqnarray*}
\E^{\pi}[\varphi(x_{t + 1})] &=& \E^{\pi}\Big[\E^{\pi}\big[\varphi(x_{t + 1}) | x_0 \ldots, x_t\big]\Big] = \E^{\pi}\Big[\sum_{x' \in \Xc} \varphi(x') P(x_t, \mu_t, u_t, \nu_t)(x')\Big]\\
&=& \sum_{x' \in \Xc} \varphi(x') \E^{\pi}\Big[ P(x_t, \mu_t, u_t, \nu_t)(x')\Big]\\
&=& \sum_{x' \in \Xc} \varphi(x') \sum_{x \in \Xc} \mu_t(x)\sum_{u \in \Uc}{\pi}_t(x, \mu_t)(u)P(x, \mu_t, u, \nu_t)(x'),
\end{eqnarray*}
which concludes that $x_{t+1}\sim\Phi(\mu_t, \pi_t(\cdot,\mu_t))$. {Here $\E^\pi$ denotes the expectation under policy $\pi$}. Therefore, under ${\pi}=\{\pi_t\}_{t=0}^{\infty}$, $\mu_{t+1}=\Phi(\mu_t, \pi_t(\cdot,\mu_t))$ defines a deterministic flow $\{\mu_t\}_{t=0}^{\infty}$ in $\Pc(\Xc)$, and $x_t\sim\mu_t$. Moreover, by Fubini's theorem
\begin{eqnarray*}
v^\pi(\mu) &=& \mathbb{E}^\pi\biggl[\sum_{t = 0}^\infty \gamma^t \tilde r(x_t, \mu_t, u_t, \nu_t) \bigg| x_0 \sim \mu \biggl] = \sum_{t = 0}^\infty \gamma^t\mathbb{E}^\pi\biggl[\tilde r(x_t, \mu_t, u_t, \nu_t) \bigg| x_0 \sim \mu\biggl]\\
&=& \sum_{t = 0}^\infty \gamma^t\mathbb{E}\biggl[\tilde r(x_t, \mu_t, u_t, \nu_t) \bigg| x_t \sim \mu_t, u_t\sim\pi_t(x_t, \mu_t)\biggl]\\
&=& \sum_{t = 0}^\infty \gamma^t \sum_{x \in \mathcal{X}} \sum_{u \in \mathcal{U}} \tilde r(x, \mu_t, u, {\nu}(\mu_t, \pi_t(\cdot,\mu_t))) \mu_t(x) \pi_t(x,\mu_t)(u)\\
&=& \sum_{t=0}^{\infty} \gamma^{t} r(\mu_t, \pi_t(\cdot, \mu_t)).
\end{eqnarray*}
This proves \eqref{equ:reformv}.
\end{proof}

\begin{proof}[Proof of Lemma \ref{lemma:cont_nu} ]
    \begin{eqnarray*}
        \|\nu(\mu,h)-\nu(\mu',h')\|_1&\leq&\|\nu(\mu,h)-\nu(\mu,h')\|_1+\|\nu(\mu,h')-\nu(\mu',h')\|_1\\
        &\leq&\Big|\Big|\sum_{x\in\Xc}(h(x)-h'(x))\mu(x)\Big|\Big|_1+\Big|\Big|\sum_{x\in\Xc}(\mu(x)-\mu'(x))h'(x)\Big|\Big|_1\\
        &\leq&\sum_{x\in\Xc}\mu(x)\Big|\Big|h(x)-h'(x)\Big|\Big|_1+\Big|\Big|\sum_{x\in\Xc}(\mu(x)-\mu'(x))h'(x)\Big|\Big|_1\\
        &\leq&\max_{x\in\Xc}\Big|\Big|h(x)-h'(x)\Big|\Big|_1 +\sum_{u\in\Uc}\sum_{x\in\Xc}|\mu(x)-\mu'(x)|h'(x)(u)\\
        &=&\,d_{\Hc}(h,h')+\|\mu-\mu'\|_1=d_\Cc((\mu,h),(\mu',h')).
    \end{eqnarray*}
\end{proof}

\begin{proof}[Proof of Lemma \ref{lemma:cont_r}]
    \begin{eqnarray*}
        & &|r(\mu,h)-r(\mu',h')|\\
        &=&\Big|\sum_{x\in\mathcal{X}}\sum_{u\in\mathcal{U}}\tilde r(x,\mu,u, {\nu}(\mu, h))\mu(x)h(x)(u)-\sum_{x\in\mathcal{X}}\sum_{u\in\mathcal{U}}\tilde r(x,\mu',u,{\nu}(\mu',h'))\mu'(x)h'(x)(u)\Big|\\
        & &(\text{For simplicity, denote }\tilde{r}_{x,u}=\tilde{r}(x,\mu,u, {\nu}(\mu, h)),\tilde{r}'_{x,u}=\tilde{r}(x,\mu',u, {\nu}(\mu', h')).)\\
        &\leq&\Big|\sum_{x\in\Xc}\sum_{u\in\Uc}(\tilde{r}_{x,u}-\tilde{r}'_{x,u})\mu(x)h(x)(u)\Big|+\Big|\sum_{x\in\Xc}\sum_{u\in\Uc}\tilde{r}'_{x,u}(\mu(x)h(x)(u)-\mu'(x)h'(x)(u))\Big|.
    \end{eqnarray*}
    By Assumption~\ref{ass:r_MFC} and Lemma~\ref{lemma:cont_nu}, for any $x \in \Xc, u \in \Uc$,
    \begin{eqnarray*}
        |\tilde{r}_{x,u}-\tilde{r}'_{x,u}|&\leq&L_{\tilde{r}}(\|\mu-\mu'\|_1+ \|\nu(\mu,h),\nu(\mu',h')\|_1)\\
        &\leq&L_{\tilde{r}} \cdot (\|\mu-\mu'\|_1 +d_\Cc((\mu,h),(\mu',h')))\leq2L_{\tilde{r}}d_\Cc((\mu,h),(\mu',h')).
    \end{eqnarray*}
    Meanwhile, 
    \begin{eqnarray*}
        &&\sum_{x\in\Xc}\sum_{u\in\Uc}|\mu(x)h(x)(u)-\mu'(x)h'(x)(u)|\\
        &\leq&\sum_{x\in\Xc}\sum_{u\in\Uc}|\mu(x)-\mu'(x)|h(x)(u)+\sum_{x\in\Xc}\sum_{u\in\Uc}\mu'(x)|h(x)(u)-h'(x)(u)|\\
        &=&\sum_{x\in\Xc}|\mu(x)-\mu'(x)|+\sum_{x\in\Xc}\mu'(x)\|h(x)-h'(x)\|_1\\
        &\leq&\|\mu-\mu'\|_1+ \max_{x\in\Xc}\|h_1(x)-h_2(x)\|_1=d_\Cc((\mu,h),(\mu',h')).
    \end{eqnarray*}
    Combining all these results, we have
    \begin{eqnarray*}
        |r(\mu,h)-r(\mu',h')|&\leq&\sum_{x\in\Xc}\sum_{u\in\Uc}|\tilde{r}_{x,u}-\tilde{r}'_{x,u}|\mu(x)h(x)(u)+\tilde{R}\sum_{x\in\Xc}\sum_{u\in\Uc}|\mu(x)h(x)(u)-\mu'(x)h'(x)(u)|\\
        &\leq&(\tilde{R}+2L_{\tilde{r}})d_\Cc((\mu,h),(\mu',h')).
    \end{eqnarray*}
\end{proof}

\begin{proof}[Proof of Lemma \ref{lemma:cont_phi}]
    \begin{eqnarray*}
        &&\|\Phi(\mu,h)-\Phi(\mu',h')\|_1\\
        &=&\Big|\Big|\sum_{x\in\Xc}\sum_{u\in\Uc}P(x,\mu,u,{\nu}(\mu,h))\mu(x)h(x)(u)-\sum_{x\in\Xc}\sum_{u\in\Uc}P(x,\mu',u,{\nu}(\mu',h'))\mu'(x)h'(x)(u)\Big|\Big|_1\\
        &&(\text{For simplicity, denote }P_{x,u}=P(x,\mu,u, {\nu}(\mu, h)),P'_{x,u}=P(x,\mu',u, {\nu}(\mu', h')).)\\
        &\leq&\Big|\Big|\sum_{x\in\Xc}\sum_{u\in\Uc}(P_{x,u}-P'_{x,u})\mu(x)h(x)(u)\Big|\Big|_1+\Big|\Big|\sum_{x\in\Xc}\sum_{u\in\Uc}P'_{x,u}(\mu(x)h(x)(u)-\mu'(x)h'(x)(u))\Big|\Big|_1.
    \end{eqnarray*}
    By Assumption~\ref{ass:P_MFC} and Lemma~\ref{lemma:cont_nu}, for any  $x$ and $u$,
    \begin{eqnarray*}
        ||P_{x,u}-P'_{x,u}||_1&\leq&L_{P}\cdot (\|\mu-\mu'\|_1 + \|\nu(\mu,h)-\nu(\mu',h')\|_1)\\
        &\leq&L_{P} \cdot (\|\mu-\mu'\|_1 +d_\Cc((\mu,h),(\mu',h')))\leq 2 L_{P} \cdot d_\Cc((\mu,h),(\mu',h')).
    \end{eqnarray*}
Meanwhile, from the proof of Lemma ~\ref{lemma:cont_r}, we know
    \begin{eqnarray*}
        \sum_{x\in\Xc}\sum_{u\in\Uc}|\mu(x)h(x)(u)-\mu'(x)h'(x)(u)|\leq d_\Cc((\mu,h),(\mu',h')).
    \end{eqnarray*}
    Combining all these results, we have
    \begin{eqnarray*}
        \|\Phi(\mu,h)-\Phi(\mu',h')\|_1
        &\leq&\sum_{x\in\Xc}\sum_{u\in\Uc}||P_{x,u}-P'_{x,u}||_1\mu(x)h(x,u)+\sum_{x\in\Xc}\sum_{u\in\Uc}||P'_{x,u}||_1|\mu(x)h(x)(u)-\mu'(x)h'(x)(u))|\\
        &\leq&(2L_{P}+1)d_\Cc((\mu,h),(\mu',h')).
    \end{eqnarray*}
\end{proof}

\begin{proof}[Proof of Lemma \ref{continuity-Qc}] To prove the continuity of $Q$, first fix $c$ and $c'$ $\in$ $\mathcal{C}$. Then there exists some policy $\pi$ such that $Q(c)-Q^{\pi}(c)<\frac{\epsilon}{2}$. 
    Let $c=(\mu_0,h_0),(\mu_1,h_1),(\mu_2,h_2),\dots,(\mu_t,h_t),\dots$ be the trajectory of the system starting from $c$ and then taking the policy $\pi$. Then $Q^{\pi}(c)=\sum_{t=0}^\infty\gamma^tr(\mu_t,h_t)$.
    
    Now consider the trajectory of the system starting from $c'$ and then taking $h_1,\dots,h_t,\dots$, denoted by $c'=(\mu'_{0},h'_{0}),(\mu'_{1},h_1),(\mu'_{2},h_2),\dots,(\mu'_{t},h_t),\dots$. Note that this trajectory starting from $c'$ may not be the optimal trajectory, therefore, 
    $Q(c')\geq\sum_{t=0}^\infty\gamma^t r(\mu'_{t},h_{t})$.
    By Lemma \ref{lemma:cont_r} and Lemma \ref{lemma:cont_phi}, 
    \begin{eqnarray*}
        |r(\mu'_{t},h_t)-r(\mu_t,h_t)|
        &\leq&L_r\cdot d_{\Pc(\Xc)}(\mu'_{t}, \mu_t)=L_r\cdot d_{\Pc(\Xc)}(\Phi(\mu'_{t-1}, h_{t-1}), \Phi(\mu_{t-1}, h_{t-1}))\notag\\
        &\leq& L_r\cdot L_{\Phi}\cdot d_{\Pc(\Xc)}(\mu'_{t-1}, \mu_{t-1})\leq \cdots\leq L_r\cdot L_{\Phi}^{t}\cdot d_\Cc(c, c'),
    \end{eqnarray*}
    implying that 
    \begin{eqnarray*}
      Q(c)-Q(c')
        &\leq&\,\frac{\epsilon}{2}+Q^{\pi}(c)-Q(c')\leq\,\frac{\epsilon}{2}+(r(c)-r(c'))+\sum_{t=1}^\infty\gamma^t(r(\mu_t,h_t) - r(\mu'_{t},h_{t}))\nonumber\\
        &\leq&\frac{\epsilon}{2}+\sum_{t=0}^\infty\gamma^t\cdot L_{\Phi}^t\cdot L_r\cdot d_\Cc(c, c') = \frac{\epsilon}{2}+\frac{L_r}{1-\gamma \cdot L_{\Phi}}\cdot d_\Cc(c, c').
    \end{eqnarray*}
    Similarly, one can show $Q(c')-Q(c)\leq\frac{\epsilon}{2}+\frac{L_r}{1-\gamma \cdot L_{\Phi}}\cdot d_\Cc(c, c')$. Therefore, as long as $d_\Cc(c, c')\leq\frac{\epsilon\cdot(1-\gamma \cdot L_{\Phi})}{2L_r}$, $|Q(c')-Q(c)|\leq\epsilon$. This proves that $Q$ is continuous.
\end{proof}

\begin{proof}[Proof of Lemma \ref{lemma:Contra_Q}]
 By definition, it is easy to show that $B$ and $B_{\Hc_{\epsilon}}$ map $\{f\in\R^{\Cc}:\|f\|_\infty\leq V_{\max}\}$ to itself, $B_\epsilon$ and $\widehat{B}_\epsilon$ map $\{f\in\R^{\Cc_\epsilon}:\|f\|_\infty\leq V_{\max}\}$ to itself, and $T$ maps $\{f\in\R^{\Pc(\Xc)}:\|f\|_\infty\leq V_{\max}\}$ to itself.
    
    For $B_\epsilon$, we have
    \begin{eqnarray*}
        \|B_\epsilon q_1-B_\epsilon q_2\|_\infty
        &\leq & \gamma \max_{c\in\Cc_\epsilon} \max_{\tilde{h}\in\Hc_{\epsilon}}|\Gamma_K q_1(\Phi(c),\tilde{h})-\Gamma_K q_2(\Phi(c),\tilde{h})|\\
        &\leq &\gamma \max_{c\in\Cc_\epsilon} \max_{\tilde{h}\in\Hc_{\epsilon}} \sum_{i=1}^{N_\epsilon}K(c^i, (\Phi(c),\tilde{h}))|q_1(c^i)-q_2(c^i)|\leq \gamma \|q_1-q_2\|_\infty,
    \end{eqnarray*}
    where we use (\ref{equ:kernel_cond}) for the property of kernel function $K(c^i, c)$.

    Therefore, $B_\epsilon$ is a contraction mapping with modulus $\gamma<1$ under the sup norm on $\{f\in\R^{\Cc_\epsilon}:\|f\|_\infty\leq V_{\max}\}$. By Banach Fixed Point Theorem, the statement for $B_\epsilon$ holds. Similar arguments prove the statements for the other four operators.
\end{proof}

\begin{proof}[Proof of Lemma \ref{lemma:shenqi_Q}]
    Using the same DPP argument as in Theorem \ref{thm:MKC=MDP}, we can show the  value function for \eqref{mfc_objective_dis}-\eqref{mfc_dynamics_dis} is a fixed point for $T$ \eqref{equ:Bellman_tildeV} in $\{f\in\R^{\Pc(\Xc)}:\|f\|_\infty\leq V_{\max}\}$. By Lemma \ref{lemma:Contra_Q}, it coincides with $V_{\Hc_\epsilon}$.
    
    To prove \eqref{equ:shenqi_Q}, recall from Lemma \ref{lemma:Contra_Q} that $T$ is a contraction mapping with 
    modulus $\gamma$ with the supremum norm on $\{f\in\R^{\Pc(\Xc)}:\|f\|_\infty\leq V_{\max}\}$, with a
    fixed point $V_{\Hc_\epsilon}$ which is the value function of the MFC  \reff{mfc_objective_dis}-\reff{mfc_dynamics_dis}, i.e.,   \reff{mfc_objective_1} with the action space restricted to $\Hc_\epsilon$. Moreover, define $\tilde{Q}(\mu, h):=r(\mu, h)+\gamma V_{\Hc_\epsilon}(\Phi(\mu, h))$.
    Then
    \begin{eqnarray*}
        \tilde{Q}(\mu, h)&=&\,r(\mu, h)+\gamma V_{\Hc_\epsilon}(\Phi(\mu, h))\\
        &=&\,r(\mu, h)+\gamma \max_{\tilde{h}\in\Hc_\epsilon}(r(\Phi(\mu, h),\tilde{h})+\gamma V_{\Hc_\epsilon}(\Phi(\Phi(\mu, h),\tilde{h})))\\
        &=&\,r(\mu, h)+\gamma \max_{\tilde{h}\in\Hc_\epsilon}\tilde{Q}(\Phi(\mu, h), \tilde h).
    \end{eqnarray*}
    So $\tilde{Q}\in\{f\in\R^{\Cc}:\|f\|_\infty\leq V_{\max}\}$ is a fixed point of $B_{\Hc_{\epsilon}}$. By Lemma \ref{lemma:Contra_Q}, $\tilde{Q}=Q_{\Hc_\epsilon}$.
    
    Now, since $Q_{\Hc_\epsilon}$ is  the value function of the MFC problem \eqref{mfc_objective_dis}, replacing $Q$ with $Q_{\Hc_\epsilon}$ in the argument of Lemma \ref{lemma:exist_policy} and then taking $\epsilon\to 0$  yield the Lipschitz continuity of $Q_{\Hc_\epsilon}$.
\end{proof}

\begin{proof}[Proof of Lemma~\ref{lemma:covering}]
    By Markov's inequality,
    \begin{eqnarray*}
        \P(T_{\Cc,\pi}>eT)\leq\frac{\E[T_{\Cc,\pi}]}{eT}\leq\frac{1}{e}.
    \end{eqnarray*}
    Since $T_{\Cc,\pi}$ is independent of the initial state and the dynamics are Markovian,
     the probability that $\Cc_\epsilon$ has not been covered
    during any time period with length $eT$ is less or equal to $\frac{1}{e}$.
    Therefore, for any positive integer $k$, $\P(T_{\Cc,\pi}>e k T)\leq\frac{1}{e^k}$. Take $k=\log(1/\delta)$ and we get the desired result.
\end{proof}

\begin{proof}[Proof of Corollary \ref{corr:optimal_value_app}]
    From \eqref{equvalueappro}, we have for any $\mu^N =\frac{\sum_{j=1}^N 1({x^{j, N}} = x)}{N}$,
    \begin{eqnarray*}
&& v^{\pi^*} (\mu^N) - \frac{C}{\sqrt{N}}  \leq u_N^{\pi^*}(\mu^N)    \leq v^{\pi^*} (\mu^N)+ \frac{C}{\sqrt{N}},\,\,v^{\widetilde{\pi}} (\mu^N) - \frac{C}{\sqrt{N}}  \leq u_N^{\widetilde{\pi}}(\mu^N)    \leq v^{\widetilde{\pi}} (\mu^N)+ \frac{C}{\sqrt{N}}.\\
    \end{eqnarray*}
By the optimality condition, we have $v^{\widetilde{\pi}} (\mu^N) \leq v^{\pi^*} (\mu^N)$. Hence
    \begin{eqnarray}\label{ineq1}
  u_N^{\widetilde{\pi}}(\mu^N)    \leq v^{\widetilde{\pi}} (\mu^N)+ \frac{C}{\sqrt{N}} \leq v^{\pi^*} (\mu^N) + \frac{C}{\sqrt{N}} .
    \end{eqnarray}
    Similarly since $u^{\widetilde{\pi}} (\mu^N) \geq u^{\pi^*} (\mu^N)$, we have
     \begin{eqnarray}\label{ineq2}
  v_N^{\pi^*}(\mu^N)    \leq u^{\pi^*} (\mu^N)+ \frac{C}{\sqrt{N}} \leq u^{\widetilde{\pi}} (\mu^N) + \frac{C}{\sqrt{N}} .
    \end{eqnarray}
    Combining \eqref{ineq1} and \eqref{ineq2} leads to the desired result.
\end{proof}

\end{document}